\documentclass[sigconf]{acmart}

\copyrightyear{2023}
\acmYear{2023}
\acmConference[WWW '23]{Proceedings of the Web Conference 2023}{APRIL 30 - MAY 4, 2023}{Texas, USA}
\acmBooktitle{Proceedings of the Web Conference 2023 (WWW '23), APRIL 30 - MAY 4, 2023, Texas, USA}
\acmPrice{}
\acmDOI{10.XXXXX/YYYYY.3449917}
\acmISBN{978-Y-4500-YYYY-7/21/04}

\usepackage{xcolor}
\usepackage[shortlabels]{enumitem}
\setlist[enumerate]{nosep}
\usepackage{multirow}
\usepackage{float}
\usepackage{graphicx}
\usepackage{float}
\usepackage{array}
\usepackage{amsmath,lipsum}
\newlength\Origarrayrulewidth

\makeatletter
\newcommand\footnoteref[1]{\protected@xdef\@thefnmark{\ref{#1}}\@footnotemark}
\makeatother
\renewcommand{\footnotesize}{\fontsize{8pt}{6pt}\selectfont}
\usepackage{graphicx}
\graphicspath{ {./images/} }
\usepackage[ruled]{algorithm2e}
\newtheorem{theorem}{Theorem}[section]
\newtheorem{lemma}{Lemma}[section]
\makeatletter
\newcommand{\vo}{\vec{o}\@ifnextchar{^}{\,}{}}
\makeatother
\usepackage{subcaption} 

\usepackage{adjustbox}
\usepackage{multicol}
\usepackage{multirow}
\usepackage{tabularx}

\usepackage{xcolor}
\usepackage{color}
\usepackage{colortbl} 
\usepackage{xcolor}
\definecolor{asparagus}{rgb}{0.53, 0.66, 0.42}

\newcommand{\kp}{\mathcal{KP}}
\usepackage{physics}
\usepackage{mathtools}
\usepackage{amsmath}

\setcopyright{rightsretained}

\begin{document}

\settopmatter{authorsperrow=4}
\title[Can Persistent Homology provide an efficient alternative]{Can Persistent Homology provide an efficient alternative for Evaluation of Knowledge Graph Completion Methods?}
\author{Anson Bastos}
\email{cs20resch11002@iith.ac.in}
\affiliation{%
  \institution{IIT, Hyderabad}
  \country{India}
}

\author{Kuldeep Singh}
\email{kuldeep.singh1@cerence.com}
\affiliation{%
  \institution{Zerotha Research and Cerence GmbH}
   \country{Germany}
}

\author{Abhishek Nadgeri}
\email{abhishek22596@gmail.com }

\affiliation{%
  \institution{Zerotha Research and RWTH Aachen}
  \country{Germany}
}

\author{Johannes Hoffart}
\email{johannes@hoffart.ai}
\affiliation{%
   \institution{SAP}
 \country{Germany}
}

\author{Toyotaro Suzumura}
\email{suzumura@acm.org}
\affiliation{%
  \institution{The University of Tokyo}
  \country{Japan}
}

\author{Manish Singh}
\email{msingh@cse.iith.ac.in}
\affiliation{%
  \institution{IIT Hyderabad}
   \country{India}
}
\renewcommand{\shortauthors}{Bastos, et al.}





\begin{abstract}
In this paper we present a novel method, \textit{Knowledge Persistence} ($\mathcal{KP}$), for faster evaluation of Knowledge Graph (KG) completion approaches. Current ranking-based evaluation is quadratic in the size of the KG, leading to long evaluation times and consequently a high carbon footprint. 
$\mathcal{KP}$ addresses this by representing the topology of the KG completion methods through the lens of topological data analysis, concretely using persistent homology. The characteristics of persistent homology allow $\mathcal{KP}$ to evaluate the quality of the KG completion looking only at a fraction of the data. 
Experimental results on standard datasets show that the proposed metric is highly correlated with ranking metrics (Hits@N, MR, MRR). Performance evaluation shows that $\mathcal{KP}$ is computationally efficient: In some cases, the evaluation time (validation+test) of a KG completion method has been reduced from 18 hours (using Hits@10) to 27 seconds (using $\mathcal{KP}$), and on average (across methods \& data) reduces the evaluation time (validation+test) by $\approx$ \textbf{99.96}\%.
\end{abstract}

\maketitle

\section{Introduction} \label{sec:introduction}
Publicly available Knowledge Graphs (KGs) find broad applicability in several downstream tasks such as entity linking, relation extraction, fact-checking, and question answering \citep{ji2020survey,singh2018reinvent}. 
These KGs are large graph databases used to express facts in the form of relations between real-world entities and store these facts as triples (\textit{subject, relation, object}). KGs must be continuously updated because new entities might emerge or facts about entities are extended or updated. Knowledge Graph Completion (KGC) task aims to fill the missing piece of information into an incomplete triple of KG \cite{ji2020survey,gesese2019survey,bastos2021hopfe}. 

Several Knowledge Graph Embedding (KGE) approaches have been proposed to model entities and relations in vector space for missing link prediction in a KG \cite{wang2017knowledge}. KGE methods infer the connectivity patterns (symmetry, asymmetry, etc.) in the KGs by defining a scoring function to calculate the plausibility of a knowledge graph triple. 
While calculating plausibility of a KG triple $\uptau = (e_h,r,e_t)$, the predicted score by scoring function affirms the confidence of a model that entities $e_t$ and $e_h$ are linked by $r$. 

For evaluating KGE methods, ranking metrics have been widely used \cite{ji2020survey} which is based on the following criteria: given a KG triple with a missing head or tail entity, what is the ability of the KGE method to rank candidate entities averaged over triples in a held-out test set \citep{mohamed2020popularity}? 
These ranking metrics are useful as they intend to gauge the behavior of the methods in real world applications of KG completion. Since 2019, over 100 KGE articles have been published in various leading conferences and journals that use ranking metrics as evaluation protocol\footnote{\url{https://github.com/xinguoxia/KGE\#papers}}. 

\textbf{Limitations of Ranking-based Evaluation:}
The key challenge while computing ranking metrics for model evaluation is the time taken to obtain them. Since the (most of) KGE models aim to rank all the negative triples that are not present in the KG \cite{bordes2011learning,bordes2013translating}, computing these metrics takes a quadratic time in the number of entities in the KG. Moreover, the problem gets alleviated in the case of hyper-relations \cite{zhang2021meta} where more than two entities participate, leading to exponential computation time. For instance, \citet{ali2021bringing} spent 24,804 GPU hours of computation time while performing a large-scale benchmarking of KGE methods.

There are two issues with high model evaluation time. \textbf{Firstly}, \underline{efficiency at evaluation time} is not a widely-adapted criterion for assessing KGE models alongside accuracy and related measures. There are efforts to make KGE methods efficient at training time \cite{wang2021lightweight,wang2022swift}. However, these methods also use ranking-based protocols resulting in high evaluation time. \textbf{Secondly}, the need for significant computational resources for the KG completion task excludes a large group of researchers in universities/labs with restricted GPU availability. Such preliminary exclusion implicitly challenges the basic notion of various diversity and inclusion initiatives for making the Web and its related research accessible to a wider community. In past, researchers have worked extensively towards efficient Web-related technologies such as Web Crawling \cite{broder2003efficient}, Web Indexing \cite{lim2003dynamic}, RDF processing \cite{galarraga2014partout}, etc.
Hence, for the KG completion task, similar to other efficient Web-based research, there is a necessity to develop alternative evaluation protocols to reduce the computation complexity, a crucial research gap in available KGE scientific literature. 
Another critical issue in ranking metrics is that they are biased towards popular entities and such popularity bias is not captured by current evaluation metrics \cite{mohamed2020popularity}. Hence, we need a metric which is efficient than popular ranking metrics and also omits such biases. 

\textbf{Motivation and Contribution:}
In this work, we focus on addressing above-mentioned key research gaps and aim for the first study to make KGE evaluation more efficient.
We introduce \textit{Knowledge Persistence}($\mathcal{KP}$), a method for characterizing the topology of the learnt KG representations. It builds upon Topological Data Analysis \cite{wasserman2018topological} based on the concepts from Persistent Homology(PH) \cite{edelsbrunner2000topological}, which has been proven beneficial for analyzing deep networks \cite{rieck2018neural,moor2020topological}. 
PH is able to effectively capture the 
geometry of the manifold on which the representations reside whilst requiring fraction of data \cite{edelsbrunner2000topological}. This property allows to reduce the quadratic complexity of considering all the data points 
(KG triples in our case) for ranking. Another crucial fact that makes PH useful is its stability with respect to perturbations making $\mathcal{KP}$ robust to noise \cite{hensel2021survey} mitigating the issues due to the open-world problem.
Thus we use PH due to its effectiveness for limited resources and noise \cite{on_effectiveness_of_PH}. 
Concretely, the following are our key contributions:
\begin{enumerate}
    \item We propose ($\mathcal{KP}$), a novel approach along with its theoretical foundations to estimate the performance of KGE models through the lens of topological data analysis. This allows us to drastically reduce the computation factor from order of  $\mathcal{O}(|\mathcal{E}|^2)$ to $\mathcal{O}(|\mathcal{E}|)$. The code is \href{https://github.com/ansonb/Knowledge_Persistence}{\textbf{here}}.
    \item We run extensive experiments on families of KGE methods (e.g., Translation, Rotation, Bi-Linear, Factorization, Neural Network methods) using standard benchmark datasets. The experiments show that $\mathcal{KP}$ correlates well with the standard ranking metrics. Hence, $\mathcal{KP}$ could be used for faster prototyping of KGE methods and paves the way for efficient evaluation methods in this domain. 
\end{enumerate}

In the remainder of the paper, related work is in section \ref{sec:related}. Section \ref{sec:preliminaries} briefly explains the concept of persistent homology. Section \ref{sec:problem_approach} describes the proposed method. Later, section \ref{sec:experiment} shows associated empirical results and we conclude in section \ref{sec:conclusion}.

\section{Related Work} \label{sec:related}
Broadly, KG embeddings are classified into translation and semantic matching models \cite{wang2017knowledge}. Translation methods such as TransE \cite{bordes2013translating}, TransH \cite{wang2014knowledge}, TransR \cite{lin2015learning} use distance-based scoring functions. Whereas semantic matching models (e.g., ComplEx \cite{trouillon2016complex}, Distmult \cite{DBLP:journals/corr/YangYHGD14a}, RotatE \cite{sun2018rotate}) use similarity-based scoring functions. 

\citet{kadlec2017knowledge} first pointed limitations of KGE evaluation and its dependency on hyperparameter tuning. \citep{sun2020re} with exhaustive evaluation (using ranking metrics) showed issues of scoring functions of KGE methods whereas \cite{nayyeri2019toward} studied the effect of loss function of KGE performance. \citet{jain2021embeddings} studied if KGE methods capture KG semantic properties. Work in \cite{pezeshkpour2020revisiting} provides a new dataset that allows the study of calibration results for KGE models. \citet{speranskaya2020ranking} used precision and recall rather
than rankings to measure the quality of completion models. Authors proposed a new dataset containing triples
such that their completion is both possible and impossible based
on queries. However, queries were build by creating a tight dependency on such queries for the evaluation as pointed by \cite{tiwari2021revisiting}. \citet{rim2021behavioral} proposed a capability-based evaluation where the focus is to evaluate KGE methods on various dimensions such as relation symmetry, entity hierarchy, entity disambiguation, etc. \citet{mohamed2020popularity} fixed the popularity bias of ranking metrics by introducing modified ranking metrics. The geometric perspective of KGE methods was introduced by \cite{sharma2018towards}
and its correlation with task performance. \citet{berrendorf2020interpretable} suggested the adjusted mean rank to improve reciprocal rank, which is an ordinal scale. Authors do not consider the effect of negative triples available for a
given triple under evaluation. \cite{tiwari2021revisiting} propose to balances the number of negatives per triple to improve ranking metrics. Authors suggested the preparation of training/testing splits by maintaining the topology. Work in \cite{li2021efficient} proposes efficient non-sampling techniques for KG embedding training, few other initiatives improve efficiency of KGE training time \cite{wang2021lightweight,wang2022swift,wang2021mulde}, and hyperparameter search efficiency of embedding models \cite{DBLP:journals/corr/abs-2205-02460,wang2021explainable,tu2019autone}.

Overall, the literature is rich with evaluations of knowledge graph completion methods \cite{jain2020knowledge,bansal2020impact,tabacof2019probability,safavi2020codex}. However, to the best of our knowledge, extensive attempts have not been made to improve KG evaluation protocols' efficiency, i.e., to reduce run-time of widely-used ranking metrics for faster prototyping. We position our work orthogonal to existing attempts such as \cite{sharma2018towards}, \cite{tiwari2021revisiting}, \cite{mohamed2020popularity}, and \cite{rim2021behavioral}. In contrast with these attempts, our approach provides a topological perspective of the learned KG embeddings and focuses on improving the efficiency of KGE evaluations. 
\begin{figure*}[ht!]
    \centering
    \includegraphics[width=0.85\textwidth]{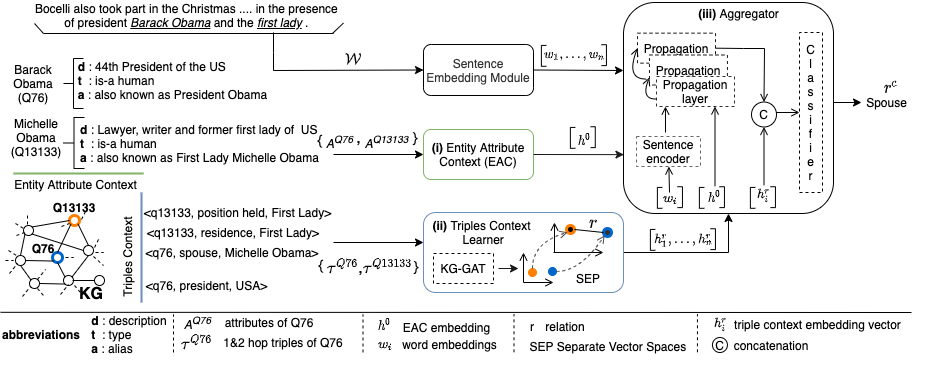}
    \caption{Calculating Knowledge Persistence($\mathcal{KP}$) score from the given KG and KG embedding method. The KG is sampled for positive($\mathcal{G}^{+}$) and negative($\mathcal{G}^{-}$) triples (step one), keeping the order $\mathcal{O}(|\mathcal{E}|)$. The edge weights represent the score obtained from the KG embedding method. In step two, the persistence diagram (PD) is computed using filtration process explained in Figure \ref{fig:pd}. In final step, a Sliced Wasserstein distance (SW) is obtained between the PDs of $\mathcal{G}^{+}$ and $\mathcal{G}^{-}$ to get the $\mathcal{KP}$ score. However, ranking metrics run the KGE methods over all the $\mathcal{O}(|\mathcal{E}|^2)$ triples as explained in bottom left part of the figure(red box).}
    \label{fig:kp_method}
\end{figure*}

\section{Preliminaries} \label{sec:preliminaries}
We now briefly describe concepts used in this paper. \\
\textbf{Ranking metrics} have been used for evaluating KG embedding methods since the inception of the KG completion task  \cite{bordes2013translating}. These metrics include the Mean Rank (MR), Mean Reciprocal Rank (MRR) and the cut-off hit ratio (Hits@N (N=1,3,10)). MR reports the average predicted rank of all the labeled triples. MRR is the average of the inverse rank of the labelled triples. Hits@N evaluates the fraction of the labeled triples that are present in the top N predicted results. 

\textbf{Persistent Homology (PH) \cite{edelsbrunner2000topological,hensel2021survey}:} 
studies the topological features such as components in 0-dimension (e.g., a node), holes in 1-dimension (e.g., a void area bounded by triangle edges) and so on, spread over a scale. 
Thus, one need not choose a scale beforehand. 
The number(rank) of these topological features(homology group) in every dimension at a particular scale can be used for downstream applications. 
Consider the simplicial complex  ( e.g., point is a 0-simplex, an edge is a 1-simplex, a triangle
is a 2-simplex ) $C$ with weights $a_0 \leq a_1 \leq a_2 \dots a_{m-1}$, which could represent the edge weights, for example, the triple score from the KG embedding method in our case. One can then define a Filtration process  \cite{edelsbrunner2000topological}, which refers to generating a nested sequence of complexes $\phi \subseteq C_1 \subseteq C_2 \subseteq \dots C_m = C$ in time/scale as the simplices below the threshold weights are added in the complex. 
The filtration process  \cite{edelsbrunner2000topological} results in the creation(birth) and destruction(death) of components, holes, etc. Thus each structure is associated with a birth-death pair $(a_i,a_j) \in R^2$ with $i \leq j$. The persistence or lifetime of each component can then be given by $a_j-a_i$. A \textit{persistence diagram (PD)} summarizes the (birth,death) pair of each object on a 2D plot, with birth times on the x axis and death times on the y axis. The points near the diagonal are shortlived components and generally are considered noise (local topology), whereas the persistent objects (global topology) are treated as features. We consider local and global topology to compare two PDs (i.e., positive and negative triple graphs in our case).

\section{Problem Statement and Method} \label{sec:problem_approach}
\subsection{Problem Setup} \label{ref:problem}
We define a KG as a tuple $KG = (\mathcal{E},\mathcal{R},\mathcal{T}^+)$ where $\mathcal{E}$ denotes the set of entities (vertices), $\mathcal{R}$ is the set of relations (edges), and $\mathcal{T}^+ \subseteq \mathcal{E} \times \mathcal{R} \times \mathcal{E} $ is a set of all triples.
A triple $\uptau = (e_h,r,e_t) \in \mathcal{T}^+$ indicates that, for the relation $r \in \mathcal{R}$, $e_h$ is the head entity (origin of the relation) while $e_t$ is the tail entity. Since $KG$ is a multigraph; $e_h=e_t$ may hold and $|\{r_{e_h,e_t}\}|\geq 0$ for any two entities. The \textit{KG completion task} predicts the entity pairs $\langle e_i,e_j\rangle$ in the KG that have a relation $r^c \in \mathcal{R}$ between them.

\subsection{Proposed Method} \label{sec:method}

\begin{figure*}[ht!]
    \centering
    \includegraphics[width=0.75\textwidth]{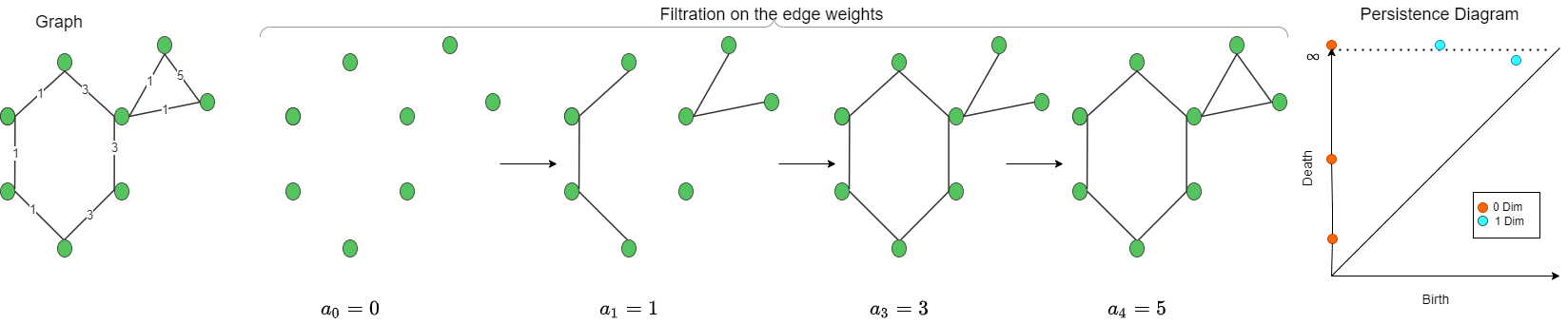}
    \caption{For a KGE method, the positive triple graph $\mathcal{G}^{+}$ is used as input (leftmost graph with edge weights) and filtration process is applied on the edge weights (calculated by KGE method) for the graph. The filtration starts with only nodes as first step, and based on the edge weights, edges are added to the nodes. The persistence diagram is given on the right with red dots indicating 0-dimensional homology (components) and the blue dots indicating 1-dimensional homology (cycles). Persistent Diagram generated from this filtration process is a condensed 2D representation of $\mathcal{G}^{+}$. A similar process is repeated for $\mathcal{G}^{-}$.}
    \label{fig:pd}
\end{figure*}

In this section we describe our approach for evaluating KG embedding methods using the theory of persistent homology (PH) . This process is divided into three steps ( Figure \ref{fig:kp_method}), namely: (i) Graph construction, (ii) Filtration process and (iii) Sliced Wasserstein distance computation. The first step creates two graphs (one for positive triples, another for negative triples) using sampling($\mathcal{O(\mathcal{V})}$ triples), with scores calculated by a KGE method as edge weights. The second step considers these graphs and, using a process called "filtration," converts to an equivalent lower dimension representation. The last step calculates the distance between graphs to provide a final metric score. We now detail the approach. 
\subsubsection{Graph Construction}
 We envisioned KGE from the topological lens while proposing an efficient solution for its evaluation. Previous works such as \cite{sharma2018towards} proposed a KGE metric only considering embedding space. However, 
we intend to preserve the topology (graph structure and its topological feature) along with the KG embedding features. We first construct graphs of positive and negative triples. We denote a graph as $(\mathcal{V},\mathcal{E})$ where $\mathcal{V}$ is the set of $N$ nodes and $\mathcal{E}$ represents the edges between them.
Consider a KG embedding method $\mathcal{M}$ that takes as input the triple $\uptau = (h, r, t) \in \mathcal{T}$ and gives the score $s_{\uptau}$ of it being a right triple. We construct a weighted directed graph $\mathcal{G}^{+}$ from positive triples $\uptau \in \mathcal{T}^{+}$ in the train set, with the entities as the nodes and the relations between them as the edges having $s_{\uptau}$ as the edge weights. Here, $s_{\uptau}$ is the score calculated by KGE method for a triple and we propose to use it as the edge weights. Our idea is to capture topology of graph ($\mathcal{G}^{+}$) with representation learned by a KG embedding method. 
We sample an order of $\mathcal{O}(|\mathcal{E}|)$ triples, $|\mathcal{E}|$ being the number of entities to keep computational time linear. Similarly, we construct a negative graph $\mathcal{G}^{-}$ by sampling the same number of unknown triples as the positive samples. One question may arise if $\mathcal{KP}$ is robust to sampling, that we answer theoretically in Theorem \ref{thm_stability_kp_ranking_main} and empirically in section \ref{sec:evaluation}.
Note, here we do not take all the negative triples in the graphs and 
consider only a fraction of what the ranking metrics need. This is a fundamental difference with ranking metrics. Ranking metrics use all the unlabeled triples as negatives for ranking, thus incurring a computational cost of $O(|\mathcal{E}|^2)$. 
\subsubsection{Filtration Process}
Having constructed the Graphs $\mathcal{G}^{+}$ and $\mathcal{G}^{-}$, we now need some definition of a distance between them to define a metric. However, since the KGs could be large with many entities and relations, directly comparing the graphs could be computationally challenging. Therefore, we allude to the theory of persistent homology (PH) to summarize the structures in the graphs in the form of the persistence diagram (PD). Such summarizing is obtained by a process known as filtration \cite{zomorodian2005computing}. One can imagine a PD as mapping of higher dimensional data
to a 2D plane upholding the representation of data points and we can then derive \textit{computational efficiency} for distance comparison between 2D representations.
Specifically, we compute the 0-dimensional topological features (i.e., connected-nodes/components) for each graph ($\mathcal{G}^{-}$ and $\mathcal{G}^{+}$) to keep the computation time linear. We also experimented using the 1-dimensional features without much empirical benefit. 

Consider the positive triple graph $\mathcal{G}^{+}$ as input (cf., Figure \ref{fig:pd}). We would need a scale (as pointed in section \ref{sec:preliminaries}) for the filtration process. Once the filtration process starts, initially, we have a graph structure containing only the nodes (entities) and no edges of $\mathcal{G}^{+}$. For capturing topological features at various scales, we define a variable $a$ which varies from $-\infty$ to $+\infty$ and it is then compared with edge weights ($s_{\uptau}$). A scale allows to capture topology at various timesteps. Thus, we use the edge weights obtained from the scores ($s_{\uptau}$) of the KGE methods for filtration. As the filtration proceeds, the graph structures (components) are generated/removed. At a given scale $a$, the graph structure ($(\mathcal{G}^{+}_{sub})_a$) contains those edges (triples) for which  $s_{\uptau} \leq a$. Formally, this is expressed as: 
\[ (\mathcal{G}^{+}_{sub})_a = \{(\mathcal{V},\mathcal{E}^{+}_a) | \mathcal{E}^{+}_a \subseteq \mathcal{E},  s_{\uptau} \leq a \ \forall \uptau \in \mathcal{E}^{+}_a\} \]
Alternatively, we  add those edges for which score of the triple is greater than or equal to the filtration value, i.e., $s_{\uptau} \geq a$ defined as 
\[ (\mathcal{G}^{+}_{super})_a = \{(\mathcal{V},\mathcal{E}^{{a}^+}) | \mathcal{E}^{{a}^+} \subseteq \mathcal{E},  s_{\uptau} \geq a \ \forall \uptau \in \mathcal{E}^{{a}^+}\} \]
One can imagine that for filtration, graph $\mathcal{G}^{+}$ is subdivided into $(\mathcal{G}^{+}_{sub})_a$ and  $(\mathcal{G}^{+}_{super})_a$ as the filtration adds/deletes edges for capturing topological features. Hence, specific components in a sub-graphs will appear and certain components will disappear at different scale levels (timesteps) $a=1,3,5$ and so on. Please note, Figure \ref{fig:pd} explains creation of PD for  $(\mathcal{G}^{+}_{sub})_a$. A similar process is repeated for $(\mathcal{G}^{+}_{super})_a$. This expansion/contraction process enables capturing topology at different time-steps without worrying about defining an optimal scale (similar to hyperparameter). 
Next step is the creation of persistent diagrams of $(\mathcal{G}^{+}_{sub})_a$ and  $(\mathcal{G}^{+}_{super})_a$ where  the x-axis and y-axis denotes the timesteps of appearance/disappearance of components. 
For creating a 2D representation graph, components of graphs which appear(disappear) during filtration process at $a_x(a_y)$ are plotted on $(a_x,a_y)$. The persistence or lifetime of each component can then be given by $a_y-a_x$. 
At implementation level, one can view  PDs($\in R^{N \times 2}$) of $(\mathcal{G}^{+}_{sub})_a$ and  $(\mathcal{G}^{+}_{super})_a$ as tensors which are concatenated into one common tensor representing positive triple graph $\mathcal{G}^{+}$.
Hence, final PD of $\mathcal{G}^{+}$ is a concatenation of PDs of $(\mathcal{G}^{+}_{sub})_a$ and $(\mathcal{G}^{+}_{super})_a$. 
This final persistent diagram represents a summary of the local and global topological features of the graph $\mathcal{G}^{+}$. Following are the benefits of a persistent diagram against considering the whole graph: 1) a 2D summary of a higher dimensional graph structure data is highly beneficial for large graphs in terms of the computational efficiency. 2) The summary could contain fewer data points than the original graph, preserving the topological information. Similarly, the process is repeated for negative triple graph $\mathcal{G}^{-}$ for creating its persistence diagram. Now, the two newly created PDs are used for calculating the proposed metric score. 
\subsubsection{Sliced Wasserstein distance computation}
To compare two PDs, generally the Wasserstein distance between them is computed \cite{fasy2020comparing}. 
As the Wasserstein distance could be computationally costly, we find the sliced Wasserstein distance \cite{pmlr-v70-carriere17a} between the PDs, which we empirically observe to be eight times faster on average. 
The Sliced Wasserstein distance($SW$) between measures $\mu$ and $\nu$ is:
\[ SW_p(\mu, \nu) = \left( \int_{S^{d-1}} W^{p}_{p}(R_{\mu}(.,\theta), R_{\nu}(.,\theta)) \right)^{\frac{1}{p}} \]
where $R_{\mu}(.,\theta)$ is the projection of $\mu$ along $\theta$, $W$ is initial Wasserstein distance. Generally a Monte Carlo average over $L$ samples is done instead of the integral. The $SW$ distance takes $\mathcal{O}(LNd + LNlog(N))$ time which can be improved to linear time $\mathcal{O}(Nd)$ for $SW_2$ (i.e., Euclidean distance) as a closed form solution \cite{nadjahi2021fast}.
Thus, 
\begin{equation}
    \mathcal{KP}(\mathcal{G}^{+},\mathcal{G}^{-}) = SW(D^{+},D^{-})
\end{equation}
where $D^{+},D^{-}$ are the persistence diagrams for $\mathcal{G}^{+},\mathcal{G}^{-}$ respectively. Since the metric is obtained by summarizing the \textit{Knowledge} graph using \textit{Persistence} diagrams we term it as \textit{Knowledge Persistence}($\mathcal{KP}$). As $\kp$ correlates well with ranking metrics (sections \ref{theory_main} and \ref{sec:experiment}), higher $\kp$ signifies a better performance of the KGE method.

\subsubsection{Theoretical justification}\label{theory_main}
This section briefly states the theoretical results justifying the proposed method to approximate the ranking metrics. We begin the analysis by assuming two distributions: One for the positive graph's edge weights(scores) and the other for the negative graph. We define a metric "PERM" (Figure \ref{fig:perm_intuition}), that is a proxy to the ranking metrics while being continuous(for the definition of integrals and derivatives) for ease of theoretical analysis. The proof sketches are given in the appendix.

\begin{figure}[h]
    \centering
    \includegraphics[width=0.40\textwidth]{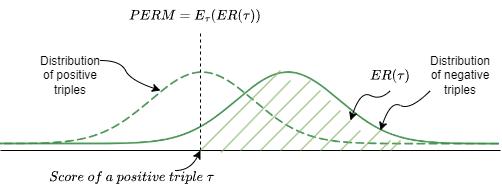}
    \caption{Figure gives an intuition of the metric PERM which is designed to be a proxy to the ranking metrics for ease of theoretical analysis. For a given positive triple $\tau$ with score $x_{\tau}$ the expected rank($ER(\tau)$) is defined as the area under the curve of the negative distribution from $x_{\tau}$ to $\infty$(shown in the shaded area above). PERM is then defined as the expectation of the expected rank under the positive distribution.}
    \label{fig:perm_intuition}
\end{figure}

\begin{definition}[Expected Ranking(ER)]
    Consider the positive triples to have the distribution $D^{+}$ and the negative triples to have the distribution $D^{-}$. For a positive triple with score $a$ its expected ranking(ER) is defined as,
    $ER(a) = \int_{x=a}^{x=\infty} D^{-}(x) dx$
\end{definition}

\begin{definition}[PERM]
    Consider the positive triples to have the distribution $D^{+}$ and the negative triples to have the distribution $D^{-}$. The PERM metric is then defined as,
    $PERM = \int_{x=-\infty}^{x=\infty} D^{+}(x)ER(x) dx$
\end{definition}
It is easy to see that PERM has a monotone increasing correspondence with the actual ranking metrics. That is, as many of the negative triples get a higher score than the positive triples, the distribution of the negative triples will shift further right of the positive distribution. Hence, the area under the curve would increase for a given triple(x=a).
We just established a monotone increasing correspondence of PERM with the ranking metrics, we now need show that there exists a one-one correspondence between PERM and $\mathcal{KP}$. 
For closed-form solutions, we work with normalised distributions (can be extended to other distributions using \cite{sakia1992box}) of KGE score under the following mild consideration:
As the KGE method converges, the mean statistic($m_{\nu}$) of the scores of the positive triples consistently lies on one side of the half-plane formed by the mean statistic($m_{\mu}$) of the negative triples, irrespective of the data distribution.
\begin{lemma}\label{lm_corr_kp_perm_main}
    $\mathcal{KP}$ has a monotone increasing correspondence with the Proxy of the Expected Ranking Metrics(PERM) under the above stated considerations as $m_{\nu}$ deviates from $m_{\mu}$
\end{lemma}
The above lemma shows that there is a one-one correspondence between $\kp$ and PERM and by definition PERM has a one-one correspondence with the ranking metrics. Therefore, the next theorem follows as a natural consequence:

\begin{theorem}\label{thm_corr_kp_ranking_main}
    $\mathcal{KP}$ has a one-one correspondence with the ranking metrics under the above stated considerations. 
\end{theorem}
The above theorem states that, with high probability, there exists a correlation between $\mathcal{KP}$ and the ranking metrics under certain considerations and proof details are in the appendix. In an ideal case,  we seek a linear relationship between the proposed measure and the ranking metric. This would help interpret whether an increase/decrease in the measure would cause a corresponding increase/decrease in the ranking metric we wish to simulate. Such interpretation becomes essential when the proposed metric has different behavior from the existing metric.
While the correlation could be high, for interpretability of the results, we would also like the change in $\mathcal{KP}$ to be bounded for a change in the scores(ranking metrics). The below theorem gives a sense for this bound.
\begin{theorem}\label{thm_stability_kp_ranking_main}
    Under the considerations of theorem \ref{thm_corr_kp_ranking_main}, the relative change in $\mathcal{KP}$ on addition of random noise to the scores is bounded by a function of the original and noise-induced covariance matrix as
    $\frac{\Delta \mathcal{KP}}{\mathcal{KP}} \leq max( (1 - |\Sigma_{\mu_1}^{+1} \Sigma_{\mu_2}^{-1}|^{\frac{3}{2}}), (1 - |\Sigma_{\nu_1}^{+1} \Sigma_{\nu_2}^{-1}|^{\frac{3}{2}}) )$,
    where $\Sigma_{\mu_1}$ and $\Sigma_{\nu_1}$ are the covariance matrices of the positive and negative triples' scores respectively and $\Sigma_{\mu_2}$ and $\Sigma_{\nu_2}$ are that of the corrupted scores.
\end{theorem}
Theorem \ref{thm_stability_kp_ranking_main} gives a bound on the change in $\mathcal{KP}$ while inducing noise in the KGE predictions. Ideally, the error/change would be 0, and as the noise is increased(and the ranking changed), gradually, the $\mathcal{KP}$ value also changes in a bounded manner as desired.

\section{Experimental Setup}\label{sec:experiment}
For de-facto KGC task (c.f., section \ref{ref:problem}), we use popular KG embedding methods from its various categories: (1) Translation: TransE \cite{bordes2013translating}, TransH \cite{wang2014knowledge}, TransR \cite{lin2015learning} (2) Bilinear, Rotation, and Factorization: RotatE \cite{sun2018rotate} TuckER \cite{balavzevic2019tucker}, and ComplEx \cite{trouillon2016complex}, (3) Neural Network based: ConvKB \cite{nguyen2018novel}. The method selection and evaluation choices are similar to \cite{mohamed2020popularity,rim2021behavioral} that propose new metrics for KG embeddings.
All methods run on a single P100 GPU machine for a maximum of 100 epochs each and evaluated every 5 epochs.
For training/testing the KG embedding methods we make use of the pykg2vec \cite{yu2021pykg2vec} library and validation runs are executed 20 times on average.
We use the standard/best hyperparameters for these datasets that the considered KGE methods reported \cite{bordes2013translating,wang2014knowledge,lin2015learning,sun2018rotate,balavzevic2019tucker,trouillon2016complex,yu2021pykg2vec}. 

\subsection{Datasets}
We use standard English KG completion datasets: WN18, WN18RR, FB15k237, FB15k, YAGO3-10 \cite{ali2021bringing,sun2018rotate}. The WN18 dataset is obtained from Wordnet \cite{miller1995wordnet} containing lexical relations between English words. WN18RR removes the inverse relations in the WN18 dataset. FB15k is obtained from the Freebase \cite{DBLP:conf/aaai/BollackerCT07} knowledge graph, and FB15k237 was created from FB15k by removing the inverse relations. The dataset details are in the Table \ref{tab:existing_datasets}. For scaling experiment, we rely on large scale YAGO3-10 dataset \cite{ali2021bringing} and due to brevity, results for Yago3-10 are in appendix ( cf., Figure \ref{fig:carbon_footprint_yago} and table \ref{tab:yago_kp_corr}). 

\subsection{Comparative Methods}
Considering ours is the first work of its kind, we select some competitive baselines as below and explain "why" we chose them. For evaluation, we report correlation \cite{chok2010pearson} between $\mathcal{KP}$ and baselines with ranking metrics (Hits@N (N= 1,3,10), MRR and MR).

\noindent \textbf{Conicity} \cite{sharma2018towards}: It finds the average cosine of the angle between an embedding and the mean embedding vector. In a sense, it gives spread of a KG embedding method in space. We would like to observe instead of topology, if calculating geometric properties of a KG embedding method be an alternative for ranking metric.

\noindent \textbf{Average Vector Length}: This metric was also proposed by \citet{sharma2018towards} to study the geometry of the KG embedding methods. It computes the average length of the embeddings.

\noindent \textbf{Graph Kernel (GK)}: we use graph kernels to compare the two graphs($\mathcal{G}^{+}, \mathcal{G}^{-}$) obtained for our approach. The rationale is to check if we could get some distance metric that correlates with the ranking metrics without persistent homology. Hence, this baseline emphasizes a direct comparison for the validity of persistent homology in our proposed method. As an implementation, we employ the widely used shortest path kernel \cite{borgwardt2005shortest} to compare how the paths(edge weights/scores) change between the two graphs. 
Since the method is computationally expensive, we sample nodes \cite{DBLP:conf/mlg/BorgwardtPVK07} and apply the kernel on the sampled graph, averaging multiple runs.
\begin{table}[ht]
\centering
\caption{(Open-Source)Benchmark Datasets for Experiments.}\label{tab:existing_datasets}
\begin{tabular}{lrrr}
\toprule
Dataset & Triples & Entities & Relations \\
\midrule
FB15K & 592,213 & 14.951 & 1,345\\
FB15K-237 & 272,115 & 14,541 & 237\\
WN18 & 151,442 & 40,943 & 18\\
WN18RR & 93,003 & 40,943 & 11\\
Yago3-10& 1,089,040 & 123,182 & 37\\
\bottomrule
\end{tabular}
\end{table}

\begin{table*}[!htb]
\begin{minipage}{1.0\linewidth}
    \medskip
\begin{adjustbox}{width=1.0\textwidth, center}
\begin{tabular}{l|ccccc|ccccc|ccccc|ccccc}
\toprule
\multicolumn{1}{c|}{\multirow{2}{*}{\textbf{Metrics}}} & \multicolumn{5}{c|}{\textbf{FB15K}} & \multicolumn{5}{c|}{\textbf{FB15K237}}  & \multicolumn{5}{c|}{\textbf{WN18}}  & \multicolumn{5}{c}{\textbf{WN18RR}}       \\
\multicolumn{1}{c|}{} & \textbf{Hits1}($\uparrow$) & \textbf{Hits3}($\uparrow$) & \textbf{Hits10}($\uparrow$) & \textbf{MR}($\downarrow$) & \textbf{MRR}($\uparrow$) & \textbf{Hits1}($\uparrow$) & \textbf{Hits3}($\uparrow$) & \textbf{Hits10}($\uparrow$) & \textbf{MR}($\downarrow$) & \textbf{MRR}($\uparrow$) & \textbf{Hits1}($\uparrow$) & \textbf{Hits3}($\uparrow$) & \textbf{Hits10}($\uparrow$) & \textbf{MR}($\downarrow$) & \textbf{MRR}($\uparrow$) & \textbf{Hits1}($\uparrow$) & \textbf{Hits3}($\uparrow$) & \textbf{Hits10}($\uparrow$) & \textbf{MR}($\downarrow$) & \textbf{MRR}($\uparrow$) \\ \midrule
          Conicity & -0.156& -0.170& -0.202& 0.085& -0.183& 0.509& 0.379& 0.356& -0.352& 0.424& -0.052& -0.096& -0.123& 0.389& -0.096& -0.267& -0.471& -0.510& 0.266& -0.448       \\
          AVL & 0.339& 0.325& 0.261& -0.423& 0.308& -0.527& -0.149& -0.158& 0.188& -0.284& 0.805& 0.825& 0.856& \cellcolor{asparagus!40}{-0.884}& 0.840& -0.272& -0.456& -0.488& 0.303& -0.438         \\
          GK(train) & -0.825& -0.852& -0.815& 0.952& -0.843& -0.903& -0.955& -0.972& 0.970& -0.965& -0.645& -0.648& -0.669& 0.611& -0.663& -0.518& -0.808& -0.840& 0.591& -0.779         \\
          GK(test) & -0.285& -0.318& -0.247& 0.629& -0.300& -0.031& -0.130& -0.123& 0.101& -0.095& -0.579& -0.565& -0.569& 0.412& -0.575& -0.276& -0.589& -0.658& 0.470& -0.549         \\ \midrule
          $\mathcal{KP}$ (Train) & 0.482& 0.418& 0.449& -0.072& 0.433& 0.773& 0.711& 0.702& -0.714& 0.745& 0.769& 0.769& 0.782& -0.682& 0.780& \cellcolor{asparagus!40}{0.500}& 0.809& 0.852& \cellcolor{asparagus!40}{-0.755}& \cellcolor{asparagus!40}{0.777}         \\
          $\mathcal{KP}$ (Test) & \cellcolor{asparagus!40}{0.786}& \cellcolor{asparagus!40}{0.731}& \cellcolor{asparagus!40}{0.661}& \cellcolor{asparagus!40}{-0.669}& \cellcolor{asparagus!40}{0.721}& \cellcolor{asparagus!40}{0.825}& \cellcolor{asparagus!40}{0.870}& \cellcolor{asparagus!40}{0.864}& \cellcolor{asparagus!40}{-0.861}& \cellcolor{asparagus!40}{0.871}& \cellcolor{asparagus!40}{0.875}& \cellcolor{asparagus!40}{0.887}& \cellcolor{asparagus!40}{0.909}& \cellcolor{asparagus!40}{-0.884}& \cellcolor{asparagus!40}{0.899}& 0.482& \cellcolor{asparagus!40}{0.816}& \cellcolor{asparagus!40}{0.863}& -0.683& 0.776         \\  \bottomrule
\end{tabular}
\end{adjustbox}
\caption{Pearson's linear correlation ($r$) scores computed from the metric scores with respect to the ranking metrics on the standard KG embedding datasets. The KG methods are evaluated after training. Green values are the best. }
\label{tab:main_table_train_test}
\end{minipage}\hfill
\end{table*}
\section{Results and Discussion} \label{sec:evaluation}
We conduct our experiments in response to the following research questions: \textbf{RQ1}: Is there a correlation between the proposed metric and ranking metrics for popular KG embedding methods? \textbf{RQ2}: Can the proposed metric be used to perform early stopping during training? \textbf{RQ3}: What is the computational efficiency of proposed metric wrt ranking metrics for KGE evaluation? 

\noindent \textbf{$\mathcal{KP}$ for faster prototyping of KGE methods:}
Our core hypothesis in the paper is to develop an efficient alternative (proxy) to the ranking metrics. Hence, for a fair evaluation, we use the triples in the test set for computing $\mathcal{KP}$.  
Ideally, this should be able to simulate the evaluation of the ranking metrics on the same (test) set. If true, there exists a high correlation between the two measures, namely the $\mathcal{KP}$ and the ranking metrics. Table \ref{tab:main_table_train_test} shows the linear correlations between the ranking metrics and our method \& baselines. We report the linear(Pearson's) correlation because we would like a linear relationship between the proposed measure and the ranking metric (for brevity, other correlations are in appendix Tables \ref{tab:rho_train_test}, \ref{tab:tau_train_test}). This would help interpret whether an increase/decrease in the measure would cause a corresponding increase/decrease in the ranking metric that we wish to simulate. Specifically we train all the KG embedding methods for a predefined number of epochs and evaluate the finally obtained models to get the ranking metrics and $\mathcal{KP}$. The correlations are then computed between $\mathcal{KP}$ and each of the ranking metrics. 
We observe that $\mathcal{KP}$(test) configuration (triples are sampled from the test set) achieves the highest correlation coefficient value among all the existing geometric and kernel baseline methods in most cases. For instance, on FB15K, $\mathcal{KP}$(test) reports high correlation value of 0.786 with Hits@1, whereas best baseline for this dataset (AVL) has corresponding correlation value as 0.339. Similarly for WN18RR, $\mathcal{KP}$(test) has correlation value of 0.482 compared to AVL with -0.272 correlation with Hits@1.
Conicity and AVL that provide geometric perspective shows mostly low positive correlation with ranking metrics whereas the Graph Kernel based method shows highly negative correlations, making these methods unsuitable for direct applicability. It indicates that the \textit{topology of the KG induced by the learnt representations} seems a good predictor of the performance on similar data distributions with high correlation with ranking metric (answering \textbf{RQ1}).\\
Furthermore, the results also report a configuration $\mathcal{KP}$(train) in which we compute $\mathcal{KP}$ on the triples of the train set and find the correlation with the ranking metrics obtained from the test set. Here our rationale is to study whether the proposed metric would be able to capture the generalizability of the unseen test (real world) data that is of a similar distribution as the training data. Initial results in Table \ref{tab:main_table_train_test} are promising with high correlation of $\mathcal{KP}$(train) with ranking metric. Hence, it may enable the use of $\mathcal{KP}$ in settings without test/validation data while using the available (possibly limited) data for training, for example, in few-shot scenarios. We leave this promising direction of research for future.

\subsection[KP as a criterion for early stopping]{$\mathcal{KP}$ as a criterion for early stopping}\label{sec:kp_early_stopping}
\textbf{Does $\mathcal{KP}$ hold correlation while early stopping?} To know when to stop the training process to prevent overfitting, we must be able to estimate the variance of the model. This is generally done by observing the validation/test set error. Thus, to use a method as a criterion for early stopping, it should be able to predict this generalization error. Table \ref{tab:main_table_train_test} explains that $\mathcal{KP}$(Train) can predict the generalizability of methods on the last epoch, it remains to empirically verify that $\mathcal{KP}$ also predicts the performance at every interval during the training process. Hence, we study the correlations of the proposed method with the ranking metrics for individual KG embedding methods in the intra-method setting. Specifically, for a given method, we obtain the $\mathcal{KP}$ score and the ranking metrics on the test set and compute the correlations at every evaluation interval. Results in Table \ref{tab:abl_intra_test} suggest that $\mathcal{KP}$ has a decent correlation in the intra-method setting. It indicates that $\mathcal{KP}$ could be used in place of the ranking metrics for deciding a criterion on early stopping if the score keeps persistently falling (answering \textbf{RQ2}).

\begin{table}[ht!]
\begin{adjustbox}{width=0.5\textwidth,center}
  \begin{tabular}{l|c|c|c|c|c|c}
    \toprule
    \multicolumn{1}{l|}{\textbf{Datasets}} &
      \multicolumn{3}{c|}{\textbf{FB15K237}} &
      \multicolumn{3}{c}{\textbf{WN18RR}} \\
      \hline
    \textbf{KG methods} & \textit{r} & $\rho$ & $\tau$ & \textit{r} & $\rho$ & $\tau$ \\
    \hline
          TransE & 0.955 & 0.861 & 0.709 & 0.876 & 0.833 & 0.722       \\
          TransH & 0.688& 0.570 & 0.409 & 0.864 & 0.717 & 0.555     \\
          TransR & 0.975& 0.942 & 0.811 & 0.954 & 0.967 & 0.889     \\
          Complex & 0.938 & 0.788 & 0.610 & 0.833 & 0.933 & 0.833       \\
          RotatE & 0.896 & 0.735 & 0.579 &  0.774 & 0.983 & 0.944     \\
          TuckER & 0.906& 0.676 & 0.527 &  0.352  & 0.25 & 0.167    \\
          ConvKB & 0.086 & 0.012 & 0.007 & 0.276 & 0.569 & 0.422     \\ \bottomrule
  \end{tabular}
  \end{adjustbox}
    \caption{Correlation scores computed between $\mathcal{KP}$ and the ranking metric(Hits@10) on the standard KG embedding datasets with the methods evaluated at every interval as the training progresses. Here, \textit{r}: Pearson correlation co-efficient, $\rho$: Spearman's correlation co-efficient, $\tau$: Kendall's Tau.}
    \label{tab:abl_intra_test}
    \vspace{-0.3cm}
\end{table}

\textbf{What is the relative error of early stopping between $\mathcal{KP}$ and Ranking Metric?}
To further cross-validate our response to \textbf{RQ2}, we now compute the absolute relative error between the ranking metrics of the best models selected by $\mathcal{KP}$ and the expected ranking metrics. Ideally, we would expect the performance of the model obtained using this process on unseen test data(preferably of the same distribution) to be close to the best achievable result, i.e., the relative error should be small. This is important as if we were to use any metric for faster prototyping, it should also be a good criterion for model selection(selecting a model with less generalization error) and being efficient. Table \ref{tab:early_stopping_results_appndx} shows that the relative error is marginal, of the order of $10^{-2}$, in most cases(with few exceptions), indicating that $\mathcal{KP}$ could be used for early stopping. The deviation is higher for some methods, such as ConvKB, which had convergence issues. We infer from observed behavior that if the KG embedding method has not converged(to good results), the correlation and, thus, the early stopping prediction may suffer. Despite a few outliers, the promising results 
shall encourage the community to research, develop, and use KGE benchmarking methods that are also computationally efficient.

\begin{table}[ht!]
\vspace{-0.3cm}
\begin{adjustbox}{width=0.5\textwidth,center}
  \begin{tabular}{l|c|c|c|c|c|c}
    \toprule
    \multicolumn{1}{l|}{\textbf{Datasets}} &
      \multicolumn{3}{c|}{\textbf{FB15K237}} &
      \multicolumn{3}{c}{\textbf{WN18RR}} \\
      \hline
    \textbf{KG methods} & hits@1 & hits@10 & MRR & hits@1 & hits@10 & MRR \\
    \hline
          TransE & 0.006 &  0.006 &  0.007 & 0.000 & 0.007  & 0.004       \\
          TransH & 0.045& 0.015 & 0.019& 0.130 & 0.018 & 0023     \\
          TransR & 0.074& 0.045 & 0.053  & 0.242 & 0.062 & 0.016     \\
          Complex & 0.001 & 0.002 & 0.003 & 0.317 & 0.021 & 0.028       \\
          RotatE & 0.022 & 0.009 & 0.007 &  0.017 & 0.005 & 0.009  \\
          TuckER & 0.008& 0.006 & 0.002 &  0.293  & 0.022  & 0.101    \\
          ConvKB & 0.000 & 0.043 & 0.043 & 0.659 & 0.453  & 0.569     \\ \bottomrule
  \end{tabular}
  \end{adjustbox}
    \caption{Early stopping using $\mathcal{KP}$. The values depict the absolute relative error between the metrics of the best models selected using $\mathcal{KP}$ and ranking metrics. }
    \label{tab:early_stopping_results_appndx}
    \vspace{-0.8cm}
\end{table}

\begin{figure} [ht]
    \centering
    \includegraphics[width=0.34\textwidth]{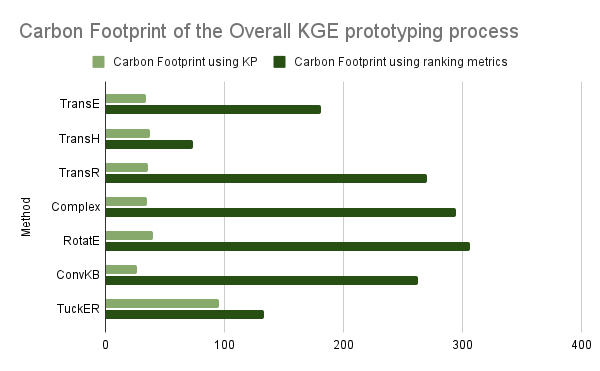}
    \caption{Figure shows a study on the carbon footprint on WN18RR when using $\mathcal{KP}$ vs Hits@10. The x-axis shows the the carbon footprint in g eq $CO_2$.}
    \label{fig:carbon_footprint}
\end{figure}
\subsection{Timing analysis and carbon footprint}
We now study the time taken for running the evaluation (including evaluation at intervals) of the same methods as in section \ref{sec:kp_early_stopping} on the standard datasets. Table \ref{tab:time_analysis_main} shows the evaluation times (validation+test)
and speedup for each method on the respective datasets. The training time is constant for ranking metric and $\mathcal{KP}$. In some cases (ConvKB), we observe $\mathcal{KP}$ achieves a speedup of up to 2576 times on model evaluation time drastically reducing evaluation time from 18 hours to 27 seconds; the latter is even
roughly equal to the carbon footprint of making a cup of coffee\footnote{\url{https://tinyurl.com/4w2xmwry}}. Furthermore, Figure \ref{fig:carbon_footprint} illustrates the carbon footprints \cite{wu2022sustainable,patterson2022carbon} of the overall process (training + evaluation) for the methods when using $\mathcal{KP}$ vs ranking metrics. Due to evaluation time drastically reduced by $\mathcal{KP}$, it also reduces overall carbon footprints. The promising results validate our attempt to develop alternative method for faster prototyping of KGE methods, thus saving carbon footprint
(answering \textbf{RQ3}).
\subsection{Ablation Studies}
We systematically provide several studies to support our evaluation and characterize different properties of $\mathcal{KP}$.

\textbf{Robustness to noise induced by sampling:} 
An important property that makes persistent homology worthwhile is its stability concerning perturbations of the filtration function. 
This means that persistent homology is robust to noise and encodes the intrinsic topological properties of the data \cite{hensel2021survey}. However, in our application of predicting the performance of KG embedding methods, one source of noise is because of sampling the negative and positive triples. It could cause perturbations in the graph topology due to the addition and deletion of edges (cf., Figure \ref{fig:pd}). Therefore, we would like the proposed metric to be stable concerning perturbations. To understand the behavior of $\mathcal{KP}$ against this noise, we conduct a study by incrementally adding samples to the graph and observing the mean and standard deviation of the correlation at each stage. In an ideal case, assuming the KG topology remains similar, the mean correlations should be in a narrow range with slight standard deviations. We observe a similar effect in Figure \ref{fig:robustness_abl} where we report the mean correlation at various fractions of triples sampled, with the standard deviation(error bands). Here, 
the mean correlation coefficients are within the range of 0.06(0.04), and the average standard deviations are about 0.02(0.02) for the FB15K237(WN18RR) dataset. 
This shows that $\mathcal{KP}$ inherits the robustness of the topological data analysis techniques, enabling linear time by sampling from the graph for dense KGs while keeping it robust.

\begin{table}[t]
\begin{adjustbox}{width=0.5\textwidth,center}
  \begin{tabular}{|l|c|c|c|c|c|c|}
    \toprule
    \multicolumn{1}{|l|}{\textbf{Metrics}} &
      \multicolumn{2}{c|}{\textbf{Hits@10}} &
      \multicolumn{2}{c|}{\textbf{$\mathcal{KP}$}} &
      \multicolumn{2}{c|}{\textbf{Speedup $\uparrow$}}\\
      \hline
   \textbf{Dataset} & \textit{FB15K237} & \textit{WN18RR} & \textit{FB15K237} & \textit{WN18RR} & \textit{FB15K237} & \textit{WN18RR} \\
    \hline
    \textbf{split} & \textit{val + test} & \textit{val + test}& \textit{val + test} & \textit{val + test} & \textit{Avg} & \textit{Avg} \\
    \hline
    TransE  & 103.6 & 86.1 & \cellcolor{asparagus!40}0.337 &\cellcolor{asparagus!40}0.120 &x 322.8 &x 754.1 \\
    \hline
    TransH  & 37.1 & 21.2 & \cellcolor{asparagus!40}0.333 &\cellcolor{asparagus!40}0.099  &x 117.0 &x 224.4 \\
    \hline
    TransR  & 192.0 & 137.1 & \cellcolor{asparagus!40}0.352 &\cellcolor{asparagus!40}0.135  &x 572.0 &x 1066.4 \\
    \hline
    Complex & 136.1 &151.4 & \cellcolor{asparagus!40}0.340  & \cellcolor{asparagus!40}0.142  & x 420.1 &x 1121.7  \\
    \hline
    RotatE & 174.2 &155.2 & \cellcolor{asparagus!40}0.359  & \cellcolor{asparagus!40}0.142  & x 509.5 &x 1145.6  \\
    \hline
    TuckER & 94.8 &22.1 & \cellcolor{asparagus!40}0.332  & \cellcolor{asparagus!40}0.098  & x 300.0 &x 241.9  \\
    \hline
    ConvKB & 1106.0 & 138.1 & \cellcolor{asparagus!40}0.451  & \cellcolor{asparagus!40}0.139  & x 2576.6 &x 1044.3 \\
    \bottomrule
  \end{tabular}
  \end{adjustbox}
    \caption{Evaluation Metric Comparison wrt Computing Time (in minutes, for 100 epochs). Column 1 denotes popular KGE methods. Depicted values denote evaluation(validation+test) time for computing a metric and corresponding speedup using $\mathcal{KP}$. $\mathcal{KP}$ significantly reduces the evaluation time (green).}
    \label{tab:time_analysis_main}
    \vspace{-0.3cm}
\end{table}

\begin{figure} [ht!]
    \centering
    \begin{subfigure}[b]{0.49\linewidth}
        \includegraphics[width=\linewidth]{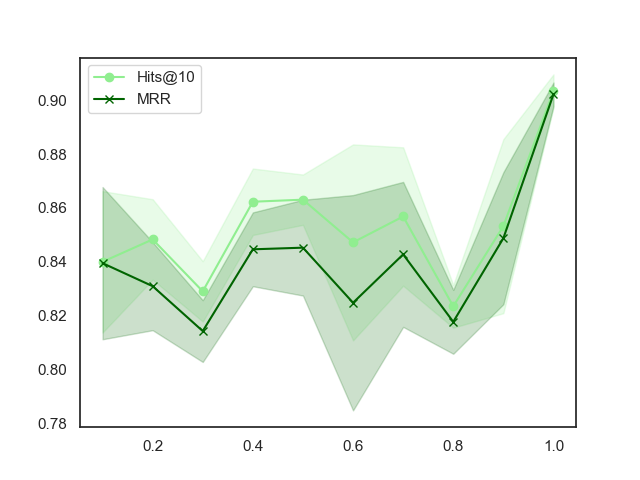}
        \label{fig:a}
    \end{subfigure} %
    \begin{subfigure}[b]{0.49\linewidth}
        \includegraphics[width=\linewidth]{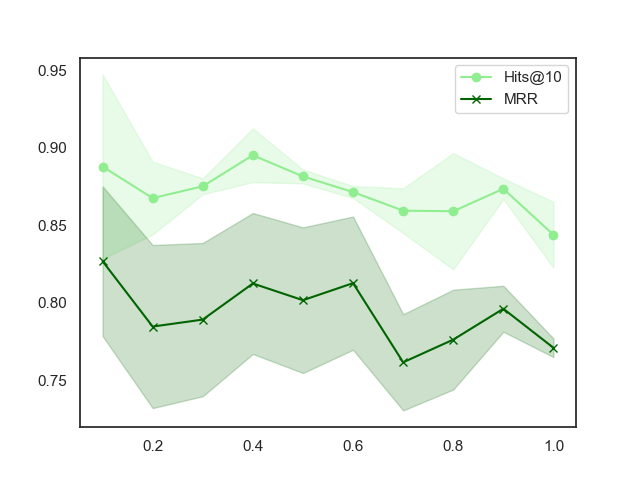}
        \label{fig:b}
    \end{subfigure} %
    \caption{Effect of sample size on the correlation coefficient between $\mathcal{KP}$ and the ranking metrics on FB15K237 (left diagram) and WN18RR datasets. The correlations for the different sampling fractions are comparable. Also, the standard deviation is less, indicating the method's robustness due to changes in local topology while doing sampling.}
    \label{fig:robustness_abl}
\end{figure}

\textbf{Generalizability Study- Correlation with Stratified Ranking Metric:}
\citet{mohamed2020popularity} proposed a new stratified metric (strat-metric) that can be tuned to focus on the unpopular entities, unlike the standard ranking metrics, using certain hyperparameters ($\beta_e \in (-1,1)$, $\beta_r \in (-1,1)$). Special cases of these hyperparameters give the micro and macro ranking metrics. Goal here is to study whether our method can predict strat-metric for the special case of $\beta_e=1, \beta_r=0$, which estimates the performance for unpopular(sparse) entities. Also, we aim to observe if $\mathcal{KP}$ holds a correlation with variants of the ranking metric concerning its generalization ability.
The results (cf., Table \ref{tab:strat_kp_corr}) shows that $\mathcal{KP}$ has a good correlation with each of the stratified ranking metrics which indicate $\mathcal{KP}$ also takes into account the local geometry/topology \cite{adams2021topology} of the sparse entities and relations.  
\begin{table}[ht!]
\begin{adjustbox}{width=0.5\textwidth,center}
  \begin{tabular}{l|c|c|c|c|c|c}
    \toprule
    \multicolumn{1}{l|}{\textbf{Datasets}} &
      \multicolumn{3}{c|}{\textbf{FB15K237}} &
      \multicolumn{3}{c}{\textbf{WN18RR}} \\
      \hline
    \textbf{Metrics} & \textit{r} & $\rho$ & $\tau$ & \textit{r} & $\rho$ & $\tau$ \\
    \hline
    Strat-Hits@1 ($\uparrow$) &0.965& 0.857& 0.714& 0.513& 0.482& 0.411 \\
    Strat-Hits@3 ($\uparrow$) &0.898& 0.821& 0.619& 0.691& 0.714& 0.524 \\
    Strat-Hits@10 ($\uparrow$) &0.871& 0.821& 0.619& 0.870& 0.750& 0.619 \\
    Strat-MR ($\downarrow$) &-0.813& -0.679& -0.524& -0.701& -0.821& -0.619 \\
    Strat-MRR ($\uparrow$) &0.806& 0.679& 0.524& 0.658& 0.714& 0.524 \\
    \bottomrule
  \end{tabular}
  \end{adjustbox}
  \caption{$\mathcal{KP}$ correlation with stratified ranking metrics as proposed in \cite{mohamed2020popularity}.}
    \label{tab:strat_kp_corr}
    \vspace{-0.3cm}
\end{table}
\subsection{Summary of Results and Open Directions}
To sum up, following are key observations gleaning from empirical studies: 1) $\mathcal{KP}$ shows high correlation with ranking metrics (Table \ref{tab:main_table_train_test}) and its stratified version (Table \ref{tab:strat_kp_corr}). It paves the way for the use of $\mathcal{KP}$ for faster prototyping of KGE methods. 2) $\mathcal{KP}$ holds a high correlation at every interval during the training process (Table \ref{tab:abl_intra_test}) with marginal relative error; hence, it could be used for early stopping of a KGE method. 3) $\mathcal{KP}$ inherits key properties of persistent homology, i.e., it is robust to noise induced by sampling. 4) The overall carbon footprints of the evaluation cycle is drastically reduced if  $\mathcal{KP}$ is preferred over ranking metrics. 

\textbf{What's Next?} We show that topological data analysis based on persistent homology can act as a proxy for ranking metrics with conclusive empirical evidence and supporting theoretical foundations. However, it is the first step toward a more extensive research agenda. We believe substantial work is needed collectively in the research community to develop strong foundations, solving scaling issues (across embedding methods, datasets, KGs, etc.) until persistent homology-based methods are widely adopted.

For example, there could be methods/datasets where the correlation turns out to be a small positive value or even negative, in which case we may not be able to use $\mathcal{KP}$ in the existing form to simulate the ranking metrics for these methods/datasets. In those cases, some alteration may exist for the same and seek further exploration similar to what stratified ranking metric \cite{mohamed2020popularity} does by fixing issues encountered in the ranking metric. Furthermore, theorem \ref{thm_stability_kp_ranking_main} would be a key to understand error bounds when interpreting limited performance (e.g., when the correlation is a small positive). However, this does not limit the use of $\mathcal{KP}$ for KGE methods as it captures and contrasts the topology of the positive and negative sampled graphs learned from these methods, which could be a useful metric by itself. 
In this paper, the emphasis is on the need for evaluation and benchmarking methods that are computationally efficient rather than providing an exhaustive one method fits all metric. We believe that there is much scope for future research in this direction. Some promising directions include 1) better sampling techniques(instead of the random sampling used in this paper), 2) rigorous theoretical analysis drawing the boundaries on the abilities/limitations across settings (zero-shot, few-shot, etc.), 3) using KP (and related metrics) in continuous spaces, that could be differentiable and approximate the ranking metrics, in the optimization process of KGE methods.

\section{Conclusion} \label{sec:conclusion}

We propose \textit{Knowledge Persistence} ($\mathcal{KP}$), first work that uses techniques from topological data analysis, as a predictor of the ranking metrics to efficiently evaluate the performance of KG embedding approaches. With theoretical and empirical evidences, our work brings efficiency at center stage in the evaluation of KG embedding methods along with traditional way of reporting their performance. 
Finally, with \underline{efficiency} as crucial criteria for evaluation, we hope KGE research becomes more inclusive and accessible to the broader research community with limited computing resources.

\textbf{Acknowledgment}
This work was partly supported by JSPS KAKENHI Grant Number JP21K21280. 
\bibliographystyle{ACM-Reference-Format}
\bibliography{bibliography}


\begin{thebibliography}{65}


\ifx \showCODEN    \undefined \def \showCODEN     #1{\unskip}     \fi
\ifx \showDOI      \undefined \def \showDOI       #1{#1}\fi
\ifx \showISBNx    \undefined \def \showISBNx     #1{\unskip}     \fi
\ifx \showISBNxiii \undefined \def \showISBNxiii  #1{\unskip}     \fi
\ifx \showISSN     \undefined \def \showISSN      #1{\unskip}     \fi
\ifx \showLCCN     \undefined \def \showLCCN      #1{\unskip}     \fi
\ifx \shownote     \undefined \def \shownote      #1{#1}          \fi
\ifx \showarticletitle \undefined \def \showarticletitle #1{#1}   \fi
\ifx \showURL      \undefined \def \showURL       {\relax}        \fi
\providecommand\bibfield[2]{#2}
\providecommand\bibinfo[2]{#2}
\providecommand\natexlab[1]{#1}
\providecommand\showeprint[2][]{arXiv:#2}

\bibitem[\protect\citeauthoryear{Adams and Moy}{Adams and Moy}{2021}]%
        {adams2021topology}
\bibfield{author}{\bibinfo{person}{Henry Adams} {and} \bibinfo{person}{Michael
  Moy}.} \bibinfo{year}{2021}\natexlab{}.
\newblock \showarticletitle{Topology Applied to Machine Learning: From Global
  to Local}.
\newblock \bibinfo{journal}{\emph{Frontiers in Artificial Intelligence}}
  \bibinfo{volume}{4} (\bibinfo{year}{2021}), \bibinfo{pages}{54}.
\newblock


\bibitem[\protect\citeauthoryear{Ali, Berrendorf, Hoyt, Vermue, Galkin,
  Sharifzadeh, Fischer, Tresp, and Lehmann}{Ali et~al\mbox{.}}{2021}]%
        {ali2021bringing}
\bibfield{author}{\bibinfo{person}{Mehdi Ali}, \bibinfo{person}{Max
  Berrendorf}, \bibinfo{person}{Charles~Tapley Hoyt}, \bibinfo{person}{Laurent
  Vermue}, \bibinfo{person}{Mikhail Galkin}, \bibinfo{person}{Sahand
  Sharifzadeh}, \bibinfo{person}{Asja Fischer}, \bibinfo{person}{Volker Tresp},
  {and} \bibinfo{person}{Jens Lehmann}.} \bibinfo{year}{2021}\natexlab{}.
\newblock \showarticletitle{Bringing light into the dark: A large-scale
  evaluation of knowledge graph embedding models under a unified framework}.
\newblock \bibinfo{journal}{\emph{IEEE Transactions on Pattern Analysis and
  Machine Intelligence}} (\bibinfo{year}{2021}).
\newblock


\bibitem[\protect\citeauthoryear{Bala{\v{z}}evi{\'c}, Allen, and
  Hospedales}{Bala{\v{z}}evi{\'c} et~al\mbox{.}}{2019}]%
        {balavzevic2019tucker}
\bibfield{author}{\bibinfo{person}{Ivana Bala{\v{z}}evi{\'c}},
  \bibinfo{person}{Carl Allen}, {and} \bibinfo{person}{Timothy Hospedales}.}
  \bibinfo{year}{2019}\natexlab{}.
\newblock \showarticletitle{TuckER: Tensor Factorization for Knowledge Graph
  Completion}. In \bibinfo{booktitle}{\emph{Proceedings of the 2019 Conference
  on Empirical Methods in Natural Language Processing and the 9th International
  Joint Conference on Natural Language Processing (EMNLP-IJCNLP)}}.
  \bibinfo{pages}{5185--5194}.
\newblock


\bibitem[\protect\citeauthoryear{Bansal, Tiwari, and Rivero}{Bansal
  et~al\mbox{.}}{2020}]%
        {bansal2020impact}
\bibfield{author}{\bibinfo{person}{Iti Bansal}, \bibinfo{person}{Sudhanshu
  Tiwari}, {and} \bibinfo{person}{Carlos~R Rivero}.}
  \bibinfo{year}{2020}\natexlab{}.
\newblock \showarticletitle{The impact of negative triple generation strategies
  and anomalies on knowledge graph completion}. In
  \bibinfo{booktitle}{\emph{Proceedings of the 29th ACM International
  Conference on Information \& Knowledge Management}}. \bibinfo{pages}{45--54}.
\newblock


\bibitem[\protect\citeauthoryear{Bastos, Singh, Nadgeri, Shekarpour, Mulang,
  and Hoffart}{Bastos et~al\mbox{.}}{2021}]%
        {bastos2021hopfe}
\bibfield{author}{\bibinfo{person}{Anson Bastos}, \bibinfo{person}{Kuldeep
  Singh}, \bibinfo{person}{Abhishek Nadgeri}, \bibinfo{person}{Saeedeh
  Shekarpour}, \bibinfo{person}{Isaiah~Onando Mulang}, {and}
  \bibinfo{person}{Johannes Hoffart}.} \bibinfo{year}{2021}\natexlab{}.
\newblock \showarticletitle{Hopfe: Knowledge graph representation learning
  using inverse hopf fibrations}. In \bibinfo{booktitle}{\emph{Proceedings of
  the 30th ACM International Conference on Information \& Knowledge
  Management}}. \bibinfo{pages}{89--99}.
\newblock


\bibitem[\protect\citeauthoryear{Berrendorf, Faerman, Vermue, and
  Tresp}{Berrendorf et~al\mbox{.}}{2020}]%
        {berrendorf2020interpretable}
\bibfield{author}{\bibinfo{person}{Max Berrendorf}, \bibinfo{person}{Evgeniy
  Faerman}, \bibinfo{person}{Laurent Vermue}, {and} \bibinfo{person}{Volker
  Tresp}.} \bibinfo{year}{2020}\natexlab{}.
\newblock \showarticletitle{Interpretable and Fair Comparison of Link
  Prediction or Entity Alignment Methods}. In \bibinfo{booktitle}{\emph{2020
  IEEE/WIC/ACM International Joint Conference on Web Intelligence and
  Intelligent Agent Technology (WI-IAT)}}. IEEE, \bibinfo{pages}{371--374}.
\newblock


\bibitem[\protect\citeauthoryear{Bollacker, Cook, and Tufts}{Bollacker
  et~al\mbox{.}}{2007}]%
        {DBLP:conf/aaai/BollackerCT07}
\bibfield{author}{\bibinfo{person}{Kurt~D. Bollacker},
  \bibinfo{person}{Robert~P. Cook}, {and} \bibinfo{person}{Patrick Tufts}.}
  \bibinfo{year}{2007}\natexlab{}.
\newblock \showarticletitle{{Freebase: A Shared Database of Structured General
  Human Knowledge}}. In \bibinfo{booktitle}{\emph{AAAI}}.
\newblock


\bibitem[\protect\citeauthoryear{Bordes, Usunier, Garcia-Duran, Weston, and
  Yakhnenko}{Bordes et~al\mbox{.}}{2013}]%
        {bordes2013translating}
\bibfield{author}{\bibinfo{person}{Antoine Bordes}, \bibinfo{person}{Nicolas
  Usunier}, \bibinfo{person}{Alberto Garcia-Duran}, \bibinfo{person}{Jason
  Weston}, {and} \bibinfo{person}{Oksana Yakhnenko}.}
  \bibinfo{year}{2013}\natexlab{}.
\newblock \showarticletitle{Translating embeddings for modeling
  multi-relational data}. In \bibinfo{booktitle}{\emph{NeurlPS}}.
  \bibinfo{pages}{1--9}.
\newblock


\bibitem[\protect\citeauthoryear{Bordes, Weston, Collobert, and Bengio}{Bordes
  et~al\mbox{.}}{2011}]%
        {bordes2011learning}
\bibfield{author}{\bibinfo{person}{Antoine Bordes}, \bibinfo{person}{Jason
  Weston}, \bibinfo{person}{Ronan Collobert}, {and} \bibinfo{person}{Yoshua
  Bengio}.} \bibinfo{year}{2011}\natexlab{}.
\newblock \showarticletitle{Learning structured embeddings of knowledge bases}.
  In \bibinfo{booktitle}{\emph{Twenty-fifth AAAI conference on artificial
  intelligence}}.
\newblock


\bibitem[\protect\citeauthoryear{Borgwardt and Kriegel}{Borgwardt and
  Kriegel}{2005}]%
        {borgwardt2005shortest}
\bibfield{author}{\bibinfo{person}{Karsten~M Borgwardt} {and}
  \bibinfo{person}{Hans-Peter Kriegel}.} \bibinfo{year}{2005}\natexlab{}.
\newblock \showarticletitle{Shortest-path kernels on graphs}. In
  \bibinfo{booktitle}{\emph{Fifth IEEE international conference on data mining
  (ICDM'05)}}. IEEE, \bibinfo{pages}{8--pp}.
\newblock


\bibitem[\protect\citeauthoryear{Borgwardt, Petri, Vishwanathan, and
  Kriegel}{Borgwardt et~al\mbox{.}}{2007}]%
        {DBLP:conf/mlg/BorgwardtPVK07}
\bibfield{author}{\bibinfo{person}{Karsten~M. Borgwardt},
  \bibinfo{person}{Tobias Petri}, \bibinfo{person}{S.~V.~N. Vishwanathan},
  {and} \bibinfo{person}{Hans{-}Peter Kriegel}.}
  \bibinfo{year}{2007}\natexlab{}.
\newblock \showarticletitle{An Efficient Sampling Scheme For Comparison of
  Large Graphs}. In \bibinfo{booktitle}{\emph{Mining and Learning with Graphs,
  {MLG}}}.
\newblock


\bibitem[\protect\citeauthoryear{Broder, Najork, and Wiener}{Broder
  et~al\mbox{.}}{2003}]%
        {broder2003efficient}
\bibfield{author}{\bibinfo{person}{Andrei~Z Broder}, \bibinfo{person}{Marc
  Najork}, {and} \bibinfo{person}{Janet~L Wiener}.}
  \bibinfo{year}{2003}\natexlab{}.
\newblock \showarticletitle{Efficient URL caching for world wide web crawling}.
  In \bibinfo{booktitle}{\emph{Proceedings of the 12th international conference
  on World Wide Web}}. \bibinfo{pages}{679--689}.
\newblock


\bibitem[\protect\citeauthoryear{Carri{\`e}re, Cuturi, and Oudot}{Carri{\`e}re
  et~al\mbox{.}}{2017}]%
        {pmlr-v70-carriere17a}
\bibfield{author}{\bibinfo{person}{Mathieu Carri{\`e}re},
  \bibinfo{person}{Marco Cuturi}, {and} \bibinfo{person}{Steve Oudot}.}
  \bibinfo{year}{2017}\natexlab{}.
\newblock \showarticletitle{Sliced {W}asserstein Kernel for Persistence
  Diagrams}. In \bibinfo{booktitle}{\emph{Proceedings of the 34th International
  Conference on Machine Learning}} \emph{(\bibinfo{series}{Proceedings of
  Machine Learning Research})}, Vol.~\bibinfo{volume}{70}.
  \bibinfo{publisher}{PMLR}, \bibinfo{pages}{664--673}.
\newblock


\bibitem[\protect\citeauthoryear{Chok}{Chok}{2010}]%
        {chok2010pearson}
\bibfield{author}{\bibinfo{person}{Nian~Shong Chok}.}
  \bibinfo{year}{2010}\natexlab{}.
\newblock \emph{\bibinfo{title}{Pearson's versus Spearman's and Kendall's
  correlation coefficients for continuous data}}.
\newblock \bibinfo{thesistype}{Ph.D. Dissertation}. \bibinfo{school}{University
  of Pittsburgh}.
\newblock


\bibitem[\protect\citeauthoryear{Edelsbrunner, Letscher, and
  Zomorodian}{Edelsbrunner et~al\mbox{.}}{2000}]%
        {edelsbrunner2000topological}
\bibfield{author}{\bibinfo{person}{Herbert Edelsbrunner},
  \bibinfo{person}{David Letscher}, {and} \bibinfo{person}{Afra Zomorodian}.}
  \bibinfo{year}{2000}\natexlab{}.
\newblock \showarticletitle{Topological persistence and simplification}. In
  \bibinfo{booktitle}{\emph{Proceedings 41st annual symposium on foundations of
  computer science}}. IEEE, \bibinfo{pages}{454--463}.
\newblock


\bibitem[\protect\citeauthoryear{Fasy, Qin, Summa, and Wenk}{Fasy
  et~al\mbox{.}}{2020}]%
        {fasy2020comparing}
\bibfield{author}{\bibinfo{person}{Brittany Fasy}, \bibinfo{person}{Yu Qin},
  \bibinfo{person}{Brian Summa}, {and} \bibinfo{person}{Carola Wenk}.}
  \bibinfo{year}{2020}\natexlab{}.
\newblock \showarticletitle{Comparing Distance Metrics on Vectorized
  Persistence Summaries}. In \bibinfo{booktitle}{\emph{NeurIPS 2020 Workshop on
  Topological Data Analysis and Beyond}}.
\newblock


\bibitem[\protect\citeauthoryear{Gal{\'a}rraga, Hose, and
  Schenkel}{Gal{\'a}rraga et~al\mbox{.}}{2014}]%
        {galarraga2014partout}
\bibfield{author}{\bibinfo{person}{Luis Gal{\'a}rraga}, \bibinfo{person}{Katja
  Hose}, {and} \bibinfo{person}{Ralf Schenkel}.}
  \bibinfo{year}{2014}\natexlab{}.
\newblock \showarticletitle{Partout: a distributed engine for efficient RDF
  processing}. In \bibinfo{booktitle}{\emph{Proceedings of the 23rd
  International Conference on World Wide Web}}. \bibinfo{pages}{267--268}.
\newblock


\bibitem[\protect\citeauthoryear{Gesese, Biswas, Alam, and Sack}{Gesese
  et~al\mbox{.}}{2019}]%
        {gesese2019survey}
\bibfield{author}{\bibinfo{person}{Genet~Asefa Gesese}, \bibinfo{person}{Russa
  Biswas}, \bibinfo{person}{Mehwish Alam}, {and} \bibinfo{person}{Harald
  Sack}.} \bibinfo{year}{2019}\natexlab{}.
\newblock \showarticletitle{A survey on knowledge graph embeddings with
  literals: Which model links better literal-ly?}
\newblock \bibinfo{journal}{\emph{Semantic Web}} \bibinfo{number}{Preprint}
  (\bibinfo{year}{2019}), \bibinfo{pages}{1--31}.
\newblock


\bibitem[\protect\citeauthoryear{Hensel, Moor, and Rieck}{Hensel
  et~al\mbox{.}}{2021}]%
        {hensel2021survey}
\bibfield{author}{\bibinfo{person}{Felix Hensel}, \bibinfo{person}{Michael
  Moor}, {and} \bibinfo{person}{Bastian Rieck}.}
  \bibinfo{year}{2021}\natexlab{}.
\newblock \showarticletitle{A survey of topological machine learning methods}.
\newblock \bibinfo{journal}{\emph{Frontiers in Artificial Intelligence}}
  \bibinfo{volume}{4} (\bibinfo{year}{2021}), \bibinfo{pages}{52}.
\newblock


\bibitem[\protect\citeauthoryear{Jain, Kalo, Balke, and Krestel}{Jain
  et~al\mbox{.}}{2021}]%
        {jain2021embeddings}
\bibfield{author}{\bibinfo{person}{Nitisha Jain},
  \bibinfo{person}{Jan-Christoph Kalo}, \bibinfo{person}{Wolf-Tilo Balke},
  {and} \bibinfo{person}{Ralf Krestel}.} \bibinfo{year}{2021}\natexlab{}.
\newblock \showarticletitle{Do Embeddings Actually Capture Knowledge Graph
  Semantics?}. In \bibinfo{booktitle}{\emph{European Semantic Web Conference}}.
  Springer, \bibinfo{pages}{143--159}.
\newblock


\bibitem[\protect\citeauthoryear{Jain, Rathi, Chakrabarti, et~al\mbox{.}}{Jain
  et~al\mbox{.}}{2020}]%
        {jain2020knowledge}
\bibfield{author}{\bibinfo{person}{Prachi Jain}, \bibinfo{person}{Sushant
  Rathi}, \bibinfo{person}{Soumen Chakrabarti}, {et~al\mbox{.}}}
  \bibinfo{year}{2020}\natexlab{}.
\newblock \showarticletitle{Knowledge base completion: Baseline strikes back
  (again)}.
\newblock \bibinfo{journal}{\emph{arXiv preprint arXiv:2005.00804}}
  (\bibinfo{year}{2020}).
\newblock


\bibitem[\protect\citeauthoryear{Ji, Pan, Cambria, Marttinen, and Yu}{Ji
  et~al\mbox{.}}{2021}]%
        {ji2020survey}
\bibfield{author}{\bibinfo{person}{Shaoxiong Ji}, \bibinfo{person}{Shirui Pan},
  \bibinfo{person}{Erik Cambria}, \bibinfo{person}{Pekka Marttinen}, {and}
  \bibinfo{person}{Philip~S Yu}.} \bibinfo{year}{2021}\natexlab{}.
\newblock \showarticletitle{A survey on knowledge graphs: Representation,
  acquisition and applications}.
\newblock \bibinfo{journal}{\emph{EEE Transactions on Neural Networks and
  Learning Systems}} (\bibinfo{year}{2021}).
\newblock


\bibitem[\protect\citeauthoryear{Kadlec, Bajgar, and Kleindienst}{Kadlec
  et~al\mbox{.}}{2017}]%
        {kadlec2017knowledge}
\bibfield{author}{\bibinfo{person}{Rudolf Kadlec},
  \bibinfo{person}{Ond{\v{r}}ej Bajgar}, {and} \bibinfo{person}{Jan
  Kleindienst}.} \bibinfo{year}{2017}\natexlab{}.
\newblock \showarticletitle{Knowledge Base Completion: Baselines Strike Back}.
  In \bibinfo{booktitle}{\emph{Proceedings of the 2nd Workshop on
  Representation Learning for NLP}}. \bibinfo{pages}{69--74}.
\newblock


\bibitem[\protect\citeauthoryear{Li, Ji, Fu, Ge, Xu, Chen, and Zhang}{Li
  et~al\mbox{.}}{2021}]%
        {li2021efficient}
\bibfield{author}{\bibinfo{person}{Zelong Li}, \bibinfo{person}{Jianchao Ji},
  \bibinfo{person}{Zuohui Fu}, \bibinfo{person}{Yingqiang Ge},
  \bibinfo{person}{Shuyuan Xu}, \bibinfo{person}{Chong Chen}, {and}
  \bibinfo{person}{Yongfeng Zhang}.} \bibinfo{year}{2021}\natexlab{}.
\newblock \showarticletitle{Efficient non-sampling knowledge graph embedding}.
  In \bibinfo{booktitle}{\emph{Proceedings of the Web Conference 2021}}.
  \bibinfo{pages}{1727--1736}.
\newblock


\bibitem[\protect\citeauthoryear{Lim, Wang, Padmanabhan, Vitter, and
  Agarwal}{Lim et~al\mbox{.}}{2003}]%
        {lim2003dynamic}
\bibfield{author}{\bibinfo{person}{Lipyeow Lim}, \bibinfo{person}{Min Wang},
  \bibinfo{person}{Sriram Padmanabhan}, \bibinfo{person}{Jeffrey~Scott Vitter},
  {and} \bibinfo{person}{Ramesh Agarwal}.} \bibinfo{year}{2003}\natexlab{}.
\newblock \showarticletitle{Dynamic maintenance of web indexes using
  landmarks}. In \bibinfo{booktitle}{\emph{Proceedings of the 12th
  international conference on World Wide Web}}. \bibinfo{pages}{102--111}.
\newblock


\bibitem[\protect\citeauthoryear{Lin, Liu, Sun, Liu, and Zhu}{Lin
  et~al\mbox{.}}{2015}]%
        {lin2015learning}
\bibfield{author}{\bibinfo{person}{Yankai Lin}, \bibinfo{person}{Zhiyuan Liu},
  \bibinfo{person}{Maosong Sun}, \bibinfo{person}{Yang Liu}, {and}
  \bibinfo{person}{Xuan Zhu}.} \bibinfo{year}{2015}\natexlab{}.
\newblock \showarticletitle{Learning entity and relation embeddings for
  knowledge graph completion}. In \bibinfo{booktitle}{\emph{Proceedings of the
  AAAI Conference on Artificial Intelligence}}, Vol.~\bibinfo{volume}{29}.
\newblock


\bibitem[\protect\citeauthoryear{Miller}{Miller}{1995}]%
        {miller1995wordnet}
\bibfield{author}{\bibinfo{person}{George~A Miller}.}
  \bibinfo{year}{1995}\natexlab{}.
\newblock \showarticletitle{WordNet: a lexical database for English}.
\newblock \bibinfo{journal}{\emph{Commun. ACM}} \bibinfo{volume}{38},
  \bibinfo{number}{11} (\bibinfo{year}{1995}), \bibinfo{pages}{39--41}.
\newblock


\bibitem[\protect\citeauthoryear{Mohamed, Parambath, Kaoudi, and
  Aboulnaga}{Mohamed et~al\mbox{.}}{2020}]%
        {mohamed2020popularity}
\bibfield{author}{\bibinfo{person}{Aisha Mohamed}, \bibinfo{person}{Shameem
  Parambath}, \bibinfo{person}{Zoi Kaoudi}, {and} \bibinfo{person}{Ashraf
  Aboulnaga}.} \bibinfo{year}{2020}\natexlab{}.
\newblock \showarticletitle{Popularity agnostic evaluation of knowledge graph
  embeddings}. In \bibinfo{booktitle}{\emph{Conference on Uncertainty in
  Artificial Intelligence}}. PMLR, \bibinfo{pages}{1059--1068}.
\newblock


\bibitem[\protect\citeauthoryear{Moor, Horn, Rieck, and Borgwardt}{Moor
  et~al\mbox{.}}{2020}]%
        {moor2020topological}
\bibfield{author}{\bibinfo{person}{Michael Moor}, \bibinfo{person}{Max Horn},
  \bibinfo{person}{Bastian Rieck}, {and} \bibinfo{person}{Karsten Borgwardt}.}
  \bibinfo{year}{2020}\natexlab{}.
\newblock \showarticletitle{Topological autoencoders}. In
  \bibinfo{booktitle}{\emph{International conference on machine learning}}.
  PMLR, \bibinfo{pages}{7045--7054}.
\newblock


\bibitem[\protect\citeauthoryear{Nadjahi, Durmus, Jacob, Badeau, and
  Simsekli}{Nadjahi et~al\mbox{.}}{2021}]%
        {nadjahi2021fast}
\bibfield{author}{\bibinfo{person}{Kimia Nadjahi}, \bibinfo{person}{Alain
  Durmus}, \bibinfo{person}{Pierre~E Jacob}, \bibinfo{person}{Roland Badeau},
  {and} \bibinfo{person}{Umut Simsekli}.} \bibinfo{year}{2021}\natexlab{}.
\newblock \showarticletitle{Fast Approximation of the Sliced-Wasserstein
  Distance Using Concentration of Random Projections}.
\newblock \bibinfo{journal}{\emph{Advances in Neural Information Processing
  Systems}}  \bibinfo{volume}{34} (\bibinfo{year}{2021}).
\newblock


\bibitem[\protect\citeauthoryear{Nayyeri, Xu, Yaghoobzadeh, Yazdi, and
  Lehmann}{Nayyeri et~al\mbox{.}}{2019}]%
        {nayyeri2019toward}
\bibfield{author}{\bibinfo{person}{Mojtaba Nayyeri}, \bibinfo{person}{Chengjin
  Xu}, \bibinfo{person}{Yadollah Yaghoobzadeh}, \bibinfo{person}{Hamed~Shariat
  Yazdi}, {and} \bibinfo{person}{Jens Lehmann}.}
  \bibinfo{year}{2019}\natexlab{}.
\newblock \showarticletitle{Toward Understanding The Effect Of Loss function On
  Then Performance Of Knowledge Graph Embedding}.
\newblock \bibinfo{journal}{\emph{arXiv preprint arXiv:1909.00519}}
  (\bibinfo{year}{2019}).
\newblock


\bibitem[\protect\citeauthoryear{Nguyen, Nguyen, Phung, et~al\mbox{.}}{Nguyen
  et~al\mbox{.}}{2018}]%
        {nguyen2018novel}
\bibfield{author}{\bibinfo{person}{Tu~Dinh Nguyen}, \bibinfo{person}{Dat~Quoc
  Nguyen}, \bibinfo{person}{Dinh Phung}, {et~al\mbox{.}}}
  \bibinfo{year}{2018}\natexlab{}.
\newblock \showarticletitle{A Novel Embedding Model for Knowledge Base
  Completion Based on Convolutional Neural Network}. In
  \bibinfo{booktitle}{\emph{NAACL}}. \bibinfo{pages}{327--333}.
\newblock


\bibitem[\protect\citeauthoryear{Patterson, Gonzalez, H{\"o}lzle, Le, Liang,
  Munguia, Rothchild, So, Texier, and Dean}{Patterson et~al\mbox{.}}{2022}]%
        {patterson2022carbon}
\bibfield{author}{\bibinfo{person}{David Patterson}, \bibinfo{person}{Joseph
  Gonzalez}, \bibinfo{person}{Urs H{\"o}lzle}, \bibinfo{person}{Quoc Le},
  \bibinfo{person}{Chen Liang}, \bibinfo{person}{Lluis-Miquel Munguia},
  \bibinfo{person}{Daniel Rothchild}, \bibinfo{person}{David So},
  \bibinfo{person}{Maud Texier}, {and} \bibinfo{person}{Jeff Dean}.}
  \bibinfo{year}{2022}\natexlab{}.
\newblock \showarticletitle{The Carbon Footprint of Machine Learning Training
  Will Plateau, Then Shrink}.
\newblock \bibinfo{journal}{\emph{arXiv preprint arXiv:2204.05149}}
  (\bibinfo{year}{2022}).
\newblock


\bibitem[\protect\citeauthoryear{Peng, Chen, Lin, and Stevenson}{Peng
  et~al\mbox{.}}{2021}]%
        {peng2021highly}
\bibfield{author}{\bibinfo{person}{Xutan Peng}, \bibinfo{person}{Guanyi Chen},
  \bibinfo{person}{Chenghua Lin}, {and} \bibinfo{person}{Mark Stevenson}.}
  \bibinfo{year}{2021}\natexlab{}.
\newblock \showarticletitle{Highly Efficient Knowledge Graph Embedding Learning
  with Orthogonal Procrustes Analysis}. In \bibinfo{booktitle}{\emph{NAACL}}.
  \bibinfo{pages}{2364--2375}.
\newblock


\bibitem[\protect\citeauthoryear{Pezeshkpour, Tian, and Singh}{Pezeshkpour
  et~al\mbox{.}}{2020}]%
        {pezeshkpour2020revisiting}
\bibfield{author}{\bibinfo{person}{Pouya Pezeshkpour}, \bibinfo{person}{Yifan
  Tian}, {and} \bibinfo{person}{Sameer Singh}.}
  \bibinfo{year}{2020}\natexlab{}.
\newblock \showarticletitle{Revisiting evaluation of knowledge base completion
  models}. In \bibinfo{booktitle}{\emph{Automated Knowledge Base
  Construction}}.
\newblock


\bibitem[\protect\citeauthoryear{Rieck, Togninalli, Bock, Moor, Horn, Gumbsch,
  and Borgwardt}{Rieck et~al\mbox{.}}{2018}]%
        {rieck2018neural}
\bibfield{author}{\bibinfo{person}{Bastian Rieck}, \bibinfo{person}{Matteo
  Togninalli}, \bibinfo{person}{Christian Bock}, \bibinfo{person}{Michael
  Moor}, \bibinfo{person}{Max Horn}, \bibinfo{person}{Thomas Gumbsch}, {and}
  \bibinfo{person}{Karsten Borgwardt}.} \bibinfo{year}{2018}\natexlab{}.
\newblock \showarticletitle{Neural Persistence: A Complexity Measure for Deep
  Neural Networks Using Algebraic Topology}. In
  \bibinfo{booktitle}{\emph{International Conference on Learning
  Representations}}.
\newblock


\bibitem[\protect\citeauthoryear{Rim, Lawrence, Gashteovski, Niepert, and
  Okazaki}{Rim et~al\mbox{.}}{2021}]%
        {rim2021behavioral}
\bibfield{author}{\bibinfo{person}{Wiem~Ben Rim}, \bibinfo{person}{Carolin
  Lawrence}, \bibinfo{person}{Kiril Gashteovski}, \bibinfo{person}{Mathias
  Niepert}, {and} \bibinfo{person}{Naoaki Okazaki}.}
  \bibinfo{year}{2021}\natexlab{}.
\newblock \showarticletitle{Behavioral Testing of Knowledge Graph Embedding
  Models for Link Prediction}. In \bibinfo{booktitle}{\emph{3rd Conference on
  Automated Knowledge Base Construction}}.
\newblock


\bibitem[\protect\citeauthoryear{Safavi and Koutra}{Safavi and Koutra}{2020}]%
        {safavi2020codex}
\bibfield{author}{\bibinfo{person}{Tara Safavi} {and} \bibinfo{person}{Danai
  Koutra}.} \bibinfo{year}{2020}\natexlab{}.
\newblock \showarticletitle{CoDEx: A Comprehensive Knowledge Graph Completion
  Benchmark}. In \bibinfo{booktitle}{\emph{Proceedings of the 2020 Conference
  on Empirical Methods in Natural Language Processing (EMNLP)}}.
  \bibinfo{pages}{8328--8350}.
\newblock


\bibitem[\protect\citeauthoryear{Sakia}{Sakia}{1992}]%
        {sakia1992box}
\bibfield{author}{\bibinfo{person}{Remi~M Sakia}.}
  \bibinfo{year}{1992}\natexlab{}.
\newblock \showarticletitle{The Box-Cox transformation technique: a review}.
\newblock \bibinfo{journal}{\emph{Journal of the Royal Statistical Society:
  Series D (The Statistician)}} \bibinfo{volume}{41}, \bibinfo{number}{2}
  (\bibinfo{year}{1992}), \bibinfo{pages}{169--178}.
\newblock


\bibitem[\protect\citeauthoryear{Sharma, Talukdar, et~al\mbox{.}}{Sharma
  et~al\mbox{.}}{2018}]%
        {sharma2018towards}
\bibfield{author}{\bibinfo{person}{Aditya Sharma}, \bibinfo{person}{Partha
  Talukdar}, {et~al\mbox{.}}} \bibinfo{year}{2018}\natexlab{}.
\newblock \showarticletitle{Towards understanding the geometry of knowledge
  graph embeddings}. In \bibinfo{booktitle}{\emph{Proceedings of the 56th
  Annual Meeting of the Association for Computational Linguistics (Volume 1:
  Long Papers)}}. \bibinfo{pages}{122--131}.
\newblock


\bibitem[\protect\citeauthoryear{Singh, Radhakrishna, Both, Shekarpour, Lytra,
  Usbeck, Vyas, Khikmatullaev, Punjani, Lange, et~al\mbox{.}}{Singh
  et~al\mbox{.}}{2018}]%
        {singh2018reinvent}
\bibfield{author}{\bibinfo{person}{Kuldeep Singh},
  \bibinfo{person}{Arun~Sethupat Radhakrishna}, \bibinfo{person}{Andreas Both},
  \bibinfo{person}{Saeedeh Shekarpour}, \bibinfo{person}{Ioanna Lytra},
  \bibinfo{person}{Ricardo Usbeck}, \bibinfo{person}{Akhilesh Vyas},
  \bibinfo{person}{Akmal Khikmatullaev}, \bibinfo{person}{Dharmen Punjani},
  \bibinfo{person}{Christoph Lange}, {et~al\mbox{.}}}
  \bibinfo{year}{2018}\natexlab{}.
\newblock \showarticletitle{Why reinvent the wheel: Let's build question
  answering systems together}. In \bibinfo{booktitle}{\emph{Proceedings of the
  2018 world wide web conference}}. \bibinfo{pages}{1247--1256}.
\newblock


\bibitem[\protect\citeauthoryear{Spearman}{Spearman}{1907}]%
        {spearman_10.2307/1412408}
\bibfield{author}{\bibinfo{person}{C. Spearman}.}
  \bibinfo{year}{1907}\natexlab{}.
\newblock \showarticletitle{Demonstration of Formulæ for True Measurement of
  Correlation}.
\newblock \bibinfo{journal}{\emph{The American Journal of Psychology}}
  \bibinfo{volume}{18}, \bibinfo{number}{2} (\bibinfo{year}{1907}),
  \bibinfo{pages}{161--169}.
\newblock
\showISSN{00029556}
\urldef\tempurl%
\url{http://www.jstor.org/stable/1412408}
\showURL{%
\tempurl}


\bibitem[\protect\citeauthoryear{Speranskaya, Schmitt, and Roth}{Speranskaya
  et~al\mbox{.}}{2020}]%
        {speranskaya2020ranking}
\bibfield{author}{\bibinfo{person}{Marina Speranskaya}, \bibinfo{person}{Martin
  Schmitt}, {and} \bibinfo{person}{Benjamin Roth}.}
  \bibinfo{year}{2020}\natexlab{}.
\newblock \showarticletitle{Ranking vs. Classifying: Measuring Knowledge Base
  Completion Quality}. In \bibinfo{booktitle}{\emph{Automated Knowledge Base
  Construction}}.
\newblock


\bibitem[\protect\citeauthoryear{Sun, Deng, Nie, and Tang}{Sun
  et~al\mbox{.}}{2018}]%
        {sun2018rotate}
\bibfield{author}{\bibinfo{person}{Zhiqing Sun}, \bibinfo{person}{Zhi-Hong
  Deng}, \bibinfo{person}{Jian-Yun Nie}, {and} \bibinfo{person}{Jian Tang}.}
  \bibinfo{year}{2018}\natexlab{}.
\newblock \showarticletitle{RotatE: Knowledge Graph Embedding by Relational
  Rotation in Complex Space}. In \bibinfo{booktitle}{\emph{International
  Conference on Learning Representations}}.
\newblock


\bibitem[\protect\citeauthoryear{Sun, Vashishth, Sanyal, Talukdar, and
  Yang}{Sun et~al\mbox{.}}{2020}]%
        {sun2020re}
\bibfield{author}{\bibinfo{person}{Zhiqing Sun}, \bibinfo{person}{Shikhar
  Vashishth}, \bibinfo{person}{Soumya Sanyal}, \bibinfo{person}{Partha
  Talukdar}, {and} \bibinfo{person}{Yiming Yang}.}
  \bibinfo{year}{2020}\natexlab{}.
\newblock \showarticletitle{A Re-evaluation of Knowledge Graph Completion
  Methods}. In \bibinfo{booktitle}{\emph{Proceedings of the 58th Annual Meeting
  of the Association for Computational Linguistics}}.
  \bibinfo{pages}{5516--5522}.
\newblock


\bibitem[\protect\citeauthoryear{Tabacof and Costabello}{Tabacof and
  Costabello}{2019}]%
        {tabacof2019probability}
\bibfield{author}{\bibinfo{person}{Pedro Tabacof} {and} \bibinfo{person}{Luca
  Costabello}.} \bibinfo{year}{2019}\natexlab{}.
\newblock \showarticletitle{Probability Calibration for Knowledge Graph
  Embedding Models}. In \bibinfo{booktitle}{\emph{International Conference on
  Learning Representations}}.
\newblock


\bibitem[\protect\citeauthoryear{Tiwari, Bansal, and Rivero}{Tiwari
  et~al\mbox{.}}{2021}]%
        {tiwari2021revisiting}
\bibfield{author}{\bibinfo{person}{Sudhanshu Tiwari}, \bibinfo{person}{Iti
  Bansal}, {and} \bibinfo{person}{Carlos~R Rivero}.}
  \bibinfo{year}{2021}\natexlab{}.
\newblock \showarticletitle{Revisiting the evaluation protocol of knowledge
  graph completion methods for link prediction}. In
  \bibinfo{booktitle}{\emph{Proceedings of the Web Conference 2021}}.
  \bibinfo{pages}{809--820}.
\newblock


\bibitem[\protect\citeauthoryear{Trouillon, Welbl, Riedel, Gaussier, and
  Bouchard}{Trouillon et~al\mbox{.}}{2016}]%
        {trouillon2016complex}
\bibfield{author}{\bibinfo{person}{Th{\'e}o Trouillon},
  \bibinfo{person}{Johannes Welbl}, \bibinfo{person}{Sebastian Riedel},
  \bibinfo{person}{{\'E}ric Gaussier}, {and} \bibinfo{person}{Guillaume
  Bouchard}.} \bibinfo{year}{2016}\natexlab{}.
\newblock \showarticletitle{Complex embeddings for simple link prediction}. In
  \bibinfo{booktitle}{\emph{International Conference on Machine Learning}}.
  PMLR, \bibinfo{pages}{2071--2080}.
\newblock


\bibitem[\protect\citeauthoryear{Tu, Ma, Cui, Pei, and Zhu}{Tu
  et~al\mbox{.}}{2019}]%
        {tu2019autone}
\bibfield{author}{\bibinfo{person}{Ke Tu}, \bibinfo{person}{Jianxin Ma},
  \bibinfo{person}{Peng Cui}, \bibinfo{person}{Jian Pei}, {and}
  \bibinfo{person}{Wenwu Zhu}.} \bibinfo{year}{2019}\natexlab{}.
\newblock \showarticletitle{Autone: Hyperparameter optimization for massive
  network embedding}. In \bibinfo{booktitle}{\emph{Proceedings of the 25th ACM
  SIGKDD International Conference on Knowledge Discovery \& Data Mining}}.
  \bibinfo{pages}{216--225}.
\newblock


\bibitem[\protect\citeauthoryear{Turkeš, Montúfar, and Otter}{Turkeš
  et~al\mbox{.}}{2022}]%
        {on_effectiveness_of_PH}
\bibfield{author}{\bibinfo{person}{Renata Turkeš}, \bibinfo{person}{Guido
  Montúfar}, {and} \bibinfo{person}{Nina Otter}.}
  \bibinfo{year}{2022}\natexlab{}.
\newblock \bibinfo{title}{On the effectiveness of persistent homology}.
\newblock
\newblock
\urldef\tempurl%
\url{https://doi.org/10.48550/ARXIV.2206.10551}
\showDOI{\tempurl}


\bibitem[\protect\citeauthoryear{Villani}{Villani}{2009}]%
        {villani2009optimal}
\bibfield{author}{\bibinfo{person}{C. Villani}.}
  \bibinfo{year}{2009}\natexlab{}.
\newblock \showarticletitle{Optimal transport}.
\newblock \bibinfo{journal}{\emph{Grundlehren der {Mathematischen}
  {Wissenschaften} [{Fundamental} {Principles} of {Mathematical} {Sciences}]}}
  \bibinfo{volume}{338} (\bibinfo{year}{2009}).
\newblock


\bibitem[\protect\citeauthoryear{Wang, Wang, Lian, and Gao}{Wang
  et~al\mbox{.}}{2021d}]%
        {wang2021lightweight}
\bibfield{author}{\bibinfo{person}{Haoyu Wang}, \bibinfo{person}{Yaqing Wang},
  \bibinfo{person}{Defu Lian}, {and} \bibinfo{person}{Jing Gao}.}
  \bibinfo{year}{2021}\natexlab{d}.
\newblock \showarticletitle{A lightweight knowledge graph embedding framework
  for efficient inference and storage}. In
  \bibinfo{booktitle}{\emph{Proceedings of the 30th ACM International
  Conference on Information \& Knowledge Management}}.
  \bibinfo{pages}{1909--1918}.
\newblock


\bibitem[\protect\citeauthoryear{Wang, Liu, Ma, and Sheng}{Wang
  et~al\mbox{.}}{2021c}]%
        {wang2021mulde}
\bibfield{author}{\bibinfo{person}{Kai Wang}, \bibinfo{person}{Yu Liu},
  \bibinfo{person}{Qian Ma}, {and} \bibinfo{person}{Quan~Z Sheng}.}
  \bibinfo{year}{2021}\natexlab{c}.
\newblock \showarticletitle{Mulde: Multi-teacher knowledge distillation for
  low-dimensional knowledge graph embeddings}. In
  \bibinfo{booktitle}{\emph{Proceedings of the Web Conference 2021}}.
  \bibinfo{pages}{1716--1726}.
\newblock


\bibitem[\protect\citeauthoryear{Wang, Liu, and Sheng}{Wang
  et~al\mbox{.}}{2022}]%
        {wang2022swift}
\bibfield{author}{\bibinfo{person}{Kai Wang}, \bibinfo{person}{Yu Liu}, {and}
  \bibinfo{person}{Quan~Z Sheng}.} \bibinfo{year}{2022}\natexlab{}.
\newblock \showarticletitle{Swift and Sure: Hardness-aware Contrastive Learning
  for Low-dimensional Knowledge Graph Embeddings}. In
  \bibinfo{booktitle}{\emph{Proceedings of the ACM Web Conference 2022}}.
  \bibinfo{pages}{838--849}.
\newblock


\bibitem[\protect\citeauthoryear{Wang, Mao, Wang, and Guo}{Wang
  et~al\mbox{.}}{2017}]%
        {wang2017knowledge}
\bibfield{author}{\bibinfo{person}{Quan Wang}, \bibinfo{person}{Zhendong Mao},
  \bibinfo{person}{Bin Wang}, {and} \bibinfo{person}{Li Guo}.}
  \bibinfo{year}{2017}\natexlab{}.
\newblock \showarticletitle{Knowledge graph embedding: A survey of approaches
  and applications}.
\newblock \bibinfo{journal}{\emph{IEEE Transactions on Knowledge and Data
  Engineering}} \bibinfo{volume}{29}, \bibinfo{number}{12}
  (\bibinfo{year}{2017}), \bibinfo{pages}{2724--2743}.
\newblock


\bibitem[\protect\citeauthoryear{Wang, Fan, Kuang, and Zhu}{Wang
  et~al\mbox{.}}{2021a}]%
        {wang2021explainable}
\bibfield{author}{\bibinfo{person}{Xin Wang}, \bibinfo{person}{Shuyi Fan},
  \bibinfo{person}{Kun Kuang}, {and} \bibinfo{person}{Wenwu Zhu}.}
  \bibinfo{year}{2021}\natexlab{a}.
\newblock \showarticletitle{Explainable automated graph representation learning
  with hyperparameter importance}. In \bibinfo{booktitle}{\emph{International
  Conference on Machine Learning}}. PMLR, \bibinfo{pages}{10727--10737}.
\newblock


\bibitem[\protect\citeauthoryear{Wang, Gao, Zhu, Zhang, Liu, Li, and Tang}{Wang
  et~al\mbox{.}}{2021b}]%
        {wikidata_wang2021kepler}
\bibfield{author}{\bibinfo{person}{Xiaozhi Wang}, \bibinfo{person}{Tianyu Gao},
  \bibinfo{person}{Zhaocheng Zhu}, \bibinfo{person}{Zhengyan Zhang},
  \bibinfo{person}{Zhiyuan Liu}, \bibinfo{person}{Juanzi Li}, {and}
  \bibinfo{person}{Jian Tang}.} \bibinfo{year}{2021}\natexlab{b}.
\newblock \showarticletitle{KEPLER: A unified model for knowledge embedding and
  pre-trained language representation}.
\newblock \bibinfo{journal}{\emph{Transactions of the Association for
  Computational Linguistics}}  \bibinfo{volume}{9} (\bibinfo{year}{2021}),
  \bibinfo{pages}{176--194}.
\newblock


\bibitem[\protect\citeauthoryear{Wang, Zhang, Feng, and Chen}{Wang
  et~al\mbox{.}}{2014}]%
        {wang2014knowledge}
\bibfield{author}{\bibinfo{person}{Zhen Wang}, \bibinfo{person}{Jianwen Zhang},
  \bibinfo{person}{Jianlin Feng}, {and} \bibinfo{person}{Zheng Chen}.}
  \bibinfo{year}{2014}\natexlab{}.
\newblock \showarticletitle{Knowledge graph embedding by translating on
  hyperplanes}. In \bibinfo{booktitle}{\emph{Proceedings of the AAAI Conference
  on Artificial Intelligence}}, Vol.~\bibinfo{volume}{28}.
\newblock


\bibitem[\protect\citeauthoryear{Wasserman}{Wasserman}{2018}]%
        {wasserman2018topological}
\bibfield{author}{\bibinfo{person}{Larry Wasserman}.}
  \bibinfo{year}{2018}\natexlab{}.
\newblock \showarticletitle{Topological data analysis}.
\newblock \bibinfo{journal}{\emph{Annual Review of Statistics and Its
  Application}}  \bibinfo{volume}{5} (\bibinfo{year}{2018}),
  \bibinfo{pages}{501--532}.
\newblock


\bibitem[\protect\citeauthoryear{Wu, Raghavendra, Gupta, Acun, Ardalani, Maeng,
  Chang, Aga, Huang, Bai, et~al\mbox{.}}{Wu et~al\mbox{.}}{2022}]%
        {wu2022sustainable}
\bibfield{author}{\bibinfo{person}{Carole-Jean Wu}, \bibinfo{person}{Ramya
  Raghavendra}, \bibinfo{person}{Udit Gupta}, \bibinfo{person}{Bilge Acun},
  \bibinfo{person}{Newsha Ardalani}, \bibinfo{person}{Kiwan Maeng},
  \bibinfo{person}{Gloria Chang}, \bibinfo{person}{Fiona Aga},
  \bibinfo{person}{Jinshi Huang}, \bibinfo{person}{Charles Bai},
  {et~al\mbox{.}}} \bibinfo{year}{2022}\natexlab{}.
\newblock \showarticletitle{Sustainable ai: Environmental implications,
  challenges and opportunities}.
\newblock \bibinfo{journal}{\emph{Proceedings of Machine Learning and Systems}}
   \bibinfo{volume}{4} (\bibinfo{year}{2022}), \bibinfo{pages}{795--813}.
\newblock


\bibitem[\protect\citeauthoryear{Yang, Yih, He, Gao, and Deng}{Yang
  et~al\mbox{.}}{2015}]%
        {DBLP:journals/corr/YangYHGD14a}
\bibfield{author}{\bibinfo{person}{Bishan Yang}, \bibinfo{person}{Wen{-}tau
  Yih}, \bibinfo{person}{Xiaodong He}, \bibinfo{person}{Jianfeng Gao}, {and}
  \bibinfo{person}{Li Deng}.} \bibinfo{year}{2015}\natexlab{}.
\newblock \showarticletitle{Embedding Entities and Relations for Learning and
  Inference in Knowledge Bases}. In \bibinfo{booktitle}{\emph{3rd International
  Conference on Learning Representations, {ICLR}}}.
\newblock


\bibitem[\protect\citeauthoryear{Yu, Chhetri, Canedo, Goyal, and Al~Faruque}{Yu
  et~al\mbox{.}}{2021}]%
        {yu2021pykg2vec}
\bibfield{author}{\bibinfo{person}{Shih-Yuan Yu}, \bibinfo{person}{Sujit~Rokka
  Chhetri}, \bibinfo{person}{Arquimedes Canedo}, \bibinfo{person}{Palash
  Goyal}, {and} \bibinfo{person}{Mohammad~Abdullah Al~Faruque}.}
  \bibinfo{year}{2021}\natexlab{}.
\newblock \showarticletitle{Pykg2vec: A Python Library for Knowledge Graph
  Embedding.}
\newblock \bibinfo{journal}{\emph{J. Mach. Learn. Res.}}  \bibinfo{volume}{22}
  (\bibinfo{year}{2021}), \bibinfo{pages}{16--1}.
\newblock


\bibitem[\protect\citeauthoryear{Zhang, Wang, Chen, Xu, Liu, and Zhao}{Zhang
  et~al\mbox{.}}{2021}]%
        {zhang2021meta}
\bibfield{author}{\bibinfo{person}{Yufeng Zhang}, \bibinfo{person}{Weiqing
  Wang}, \bibinfo{person}{Wei Chen}, \bibinfo{person}{Jiajie Xu},
  \bibinfo{person}{An Liu}, {and} \bibinfo{person}{Lei Zhao}.}
  \bibinfo{year}{2021}\natexlab{}.
\newblock \showarticletitle{Meta-Learning Based Hyper-Relation Feature Modeling
  for Out-of-Knowledge-Base Embedding}. In
  \bibinfo{booktitle}{\emph{Proceedings of the 30th ACM International
  Conference on Information \& Knowledge Management}}.
  \bibinfo{pages}{2637--2646}.
\newblock


\bibitem[\protect\citeauthoryear{Zhang, Zhou, Yao, and Li}{Zhang
  et~al\mbox{.}}{2022}]%
        {DBLP:journals/corr/abs-2205-02460}
\bibfield{author}{\bibinfo{person}{Yongqi Zhang}, \bibinfo{person}{Zhanke
  Zhou}, \bibinfo{person}{Quanming Yao}, {and} \bibinfo{person}{Yong Li}.}
  \bibinfo{year}{2022}\natexlab{}.
\newblock \showarticletitle{KGTuner: Efficient Hyper-parameter Search for
  Knowledge Graph Learning}.
\newblock \bibinfo{journal}{\emph{CoRR}}  \bibinfo{volume}{abs/2205.02460}
  (\bibinfo{year}{2022}).
\newblock


\bibitem[\protect\citeauthoryear{Zomorodian and Carlsson}{Zomorodian and
  Carlsson}{2005}]%
        {zomorodian2005computing}
\bibfield{author}{\bibinfo{person}{Afra Zomorodian} {and}
  \bibinfo{person}{Gunnar Carlsson}.} \bibinfo{year}{2005}\natexlab{}.
\newblock \showarticletitle{Computing persistent homology}.
\newblock \bibinfo{journal}{\emph{Discrete \& Computational Geometry}}
  \bibinfo{volume}{33}, \bibinfo{number}{2} (\bibinfo{year}{2005}),
  \bibinfo{pages}{249--274}.
\newblock


\end{thebibliography}




\section{appendix}

\begin{table*}[!htb]
\vspace{0.3cm}
\begin{minipage}{1.0\linewidth}
    \medskip
\begin{adjustbox}{width=1.0\textwidth, center}
\begin{tabular}{l|ccccc|ccccc|ccccc|ccccc}
\toprule
\multicolumn{1}{c|}{\multirow{2}{*}{\textbf{Metrics}}} & \multicolumn{5}{c|}{\textbf{FB15K}} & \multicolumn{5}{c|}{\textbf{FB15K237}}  & \multicolumn{5}{c|}{\textbf{WN18}}  & \multicolumn{5}{c}{\textbf{WN18RR}}       \\
\multicolumn{1}{c|}{} & \textbf{Hits1}($\uparrow$) & \textbf{Hits3}($\uparrow$) & \textbf{Hits10}($\uparrow$) & \textbf{MR}($\downarrow$) & \textbf{MRR}($\uparrow$) & \textbf{Hits1}($\uparrow$) & \textbf{Hits3}($\uparrow$) & \textbf{Hits10}($\uparrow$) & \textbf{MR}($\downarrow$) & \textbf{MRR}($\uparrow$) & \textbf{Hits1}($\uparrow$) & \textbf{Hits3}($\uparrow$) & \textbf{Hits10}($\uparrow$) & \textbf{MR}($\downarrow$) & \textbf{MRR}($\uparrow$) & \textbf{Hits1}($\uparrow$) & \textbf{Hits3}($\uparrow$) & \textbf{Hits10}($\uparrow$) & \textbf{MR}($\downarrow$) & \textbf{MRR}($\uparrow$) \\ \midrule
          Conicity & 0.071& -0.071& 0.000& -0.036& -0.214& 0.393& 0.250& -0.071& 0.036& 0.250& 0.600& 0.600& 0.600& -0.143& 0.600& -0.393& -0.607& -0.607& 0.357& -0.607       \\
          AVL & 0.607& 0.214& 0.250& \cellcolor{asparagus!40}{-0.679}& 0.179& -0.107& -0.321& -0.429& 0.393& -0.250& 0.886& 0.886& 0.886& -0.771& 0.886& 0.321& -0.143& -0.464& 0.607& -0.143        \\
          Graph Kernel (Train) & -0.536& -0.321& -0.357& 0.964& -0.393& -0.929& -0.714& -0.607& 0.643& -0.821& -0.943& -0.943& -0.943& 0.714& -0.943& -0.357& -0.786& -0.607& 0.571& -0.786        \\ 
          Graph Kernel (Test) & -0.107& 0.107& 0.000& 0.893& 0.036& -0.429& -0.607& -0.679& 0.464& -0.607& -0.657& -0.657& -0.657& 0.086& -0.657& -0.393& -0.714& -0.821& 0.786& -0.714   \\ \midrule
          $\mathcal{KP}$ (Train) & 0.214& 0.536& 0.750& 0.000& 0.607& 0.893& \cellcolor{asparagus!40}{0.750}& 0.643& -0.679& 0.786& 0.829& 0.829& 0.829& -0.600& 0.829& 0.286& 0.714& \cellcolor{asparagus!40}{0.643}& \cellcolor{asparagus!40}{-0.750}& 0.714        \\ 
          $\mathcal{KP}$ (Test) & \cellcolor{asparagus!40}{0.964}& \cellcolor{asparagus!40}{0.750}& \cellcolor{asparagus!40}{0.750}& -0.536& \cellcolor{asparagus!40}{0.714}& \cellcolor{asparagus!40}{0.714}& \cellcolor{asparagus!40}{0.821}& \cellcolor{asparagus!40}{0.857}& \cellcolor{asparagus!40}{-0.750}& \cellcolor{asparagus!40}{0.857}& \cellcolor{asparagus!40}{0.943}& \cellcolor{asparagus!40}{0.943}& \cellcolor{asparagus!40}{0.943}& \cellcolor{asparagus!40}{-0.829}& \cellcolor{asparagus!40}{0.943}& 0.286& \cellcolor{asparagus!40}{0.714}& 0.643& \cellcolor{asparagus!40}{-0.643}& \cellcolor{asparagus!40}{0.714}        \\ \bottomrule
\end{tabular}
\end{adjustbox}
\caption{Spearman's ranked correlation ($\rho$) scores computed from the metric scores with respect to the ranking metrics on the standard KG embedding datasets. The KG methods are evaluated after training. }
\label{tab:rho_train_test}
\end{minipage}\hfill
\vspace{-0.3cm}
\end{table*}

\begin{table*}[!htb]
\vspace{0.0cm}
\begin{minipage}{1.0\linewidth}
    \medskip
\begin{adjustbox}{width=1.0\textwidth, center}
\begin{tabular}{l|ccccc|ccccc|ccccc|ccccc}
\toprule
\multicolumn{1}{c|}{\multirow{2}{*}{\textbf{Metrics}}} & \multicolumn{5}{c|}{\textbf{FB15K}} & \multicolumn{5}{c|}{\textbf{FB15K237}}  & \multicolumn{5}{c|}{\textbf{WN18}}  & \multicolumn{5}{c}{\textbf{WN18RR}}       \\
\multicolumn{1}{c|}{} & \textbf{Hits1}($\uparrow$) & \textbf{Hits3}($\uparrow$) & \textbf{Hits10}($\uparrow$) & \textbf{MR}($\downarrow$) & \textbf{MRR}($\uparrow$) & \textbf{Hits1}($\uparrow$) & \textbf{Hits3}($\uparrow$) & \textbf{Hits10}($\uparrow$) & \textbf{MR}($\downarrow$) & \textbf{MRR}($\uparrow$) & \textbf{Hits1}($\uparrow$) & \textbf{Hits3}($\uparrow$) & \textbf{Hits10}($\uparrow$) & \textbf{MR}($\downarrow$) & \textbf{MRR}($\uparrow$) & \textbf{Hits1}($\uparrow$) & \textbf{Hits3}($\uparrow$) & \textbf{Hits10}($\uparrow$) & \textbf{MR}($\downarrow$) & \textbf{MRR}($\uparrow$) \\ \midrule
          Conicity & 0.048& -0.048& -0.048& 0.048& -0.143& 0.238& 0.143& -0.048& -0.048& 0.143& 0.467& 0.467& 0.467& -0.333& 0.467& -0.238& -0.333& -0.429& 0.238& -0.333         \\
          AVL & 0.429& 0.143& 0.143& \cellcolor{asparagus!40}{-0.524}& 0.048& -0.143& -0.238& -0.429& 0.333& -0.238& 0.733& 0.733& 0.733& -0.600& 0.733& \cellcolor{asparagus!40}{0.238}& -0.048& -0.333& 0.524& -0.048        \\
          Graph Kernel (Train) & -0.429& -0.143& -0.143& 0.905& -0.238& -0.810& -0.524& -0.333& 0.429& -0.714& -0.867& -0.867& -0.867& 0.467& -0.867& -0.333& -0.619& -0.333& 0.333& -0.619        \\ 
          Graph Kernel (Test) & -0.143& 0.143& -0.048& 0.810& 0.048& -0.429& -0.524& -0.524& 0.429& -0.524& -0.600& -0.600& -0.600& 0.200& -0.600& -0.238& -0.524& -0.619& 0.619& -0.524     \\ \midrule
          $\mathcal{KP}$ (Train) & 0.143& 0.429& 0.619& -0.048& 0.524& 0.714& 0.619& 0.429& -0.524& 0.619& 0.600& 0.600& 0.600& -0.467& 0.600& \cellcolor{asparagus!40}{0.238}& \cellcolor{asparagus!40}{0.524}& \cellcolor{asparagus!40}{0.429}& \cellcolor{asparagus!40}{-0.619}& \cellcolor{asparagus!40}{0.524}      \\ 
          $\mathcal{KP}$ (Test) & \cellcolor{asparagus!40}{0.905}& \cellcolor{asparagus!40}{0.619}& \cellcolor{asparagus!40}{0.619}& -0.429& \cellcolor{asparagus!40}{0.524}& \cellcolor{asparagus!40}{0.619}& \cellcolor{asparagus!40}{0.714}& \cellcolor{asparagus!40}{0.714}& \cellcolor{asparagus!40}{-0.619}& \cellcolor{asparagus!40}{0.714}& \cellcolor{asparagus!40}{0.867}& \cellcolor{asparagus!40}{0.867}& \cellcolor{asparagus!40}{0.867}& \cellcolor{asparagus!40}{-0.733}& \cellcolor{asparagus!40}{0.867}& \cellcolor{asparagus!40}{0.238}& \cellcolor{asparagus!40}{0.524}& \cellcolor{asparagus!40}{0.429}& -0.429& \cellcolor{asparagus!40}{0.524}       \\ \bottomrule
\end{tabular}
\end{adjustbox}
\caption{Kendall's tau ($\tau$) scores computed from the metric scores with respect to the ranking metrics on the standard KG embedding datasets. The KG methods are evaluated after training. }
\label{tab:tau_train_test}
\end{minipage}\hfill
\vspace{-0.0cm}
\end{table*}

\begin{figure}[!htb]
    \centering
    \includegraphics[width=0.35\textwidth]{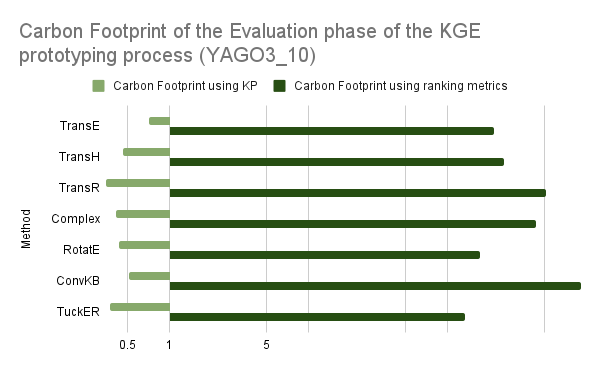}
    \caption{Study on the carbon footprint of the evaluation phase of the KGE methods on YAGO3-10 when using $\mathcal{KP}$ vs Hits@10. The x-axis shows the the carbon footprint in g eq $CO_2$ in log scale.}
    \label{fig:carbon_footprint_yago}
\end{figure}

\begin{figure}[!htb]
    \centering
    \includegraphics[width=0.35\textwidth]{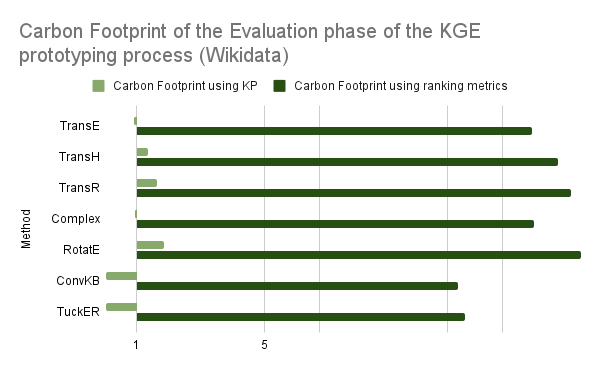}
    \caption{Study on the carbon footprint of the evaluation phase of the KGE methods on Wikidata when using $\mathcal{KP}$ vs Hits@10. The x-axis shows the the carbon footprint in g eq $CO_2$ in log scale.}
    \label{fig:carbon_footprint_wikidata}
\end{figure}

\subsection{Extended Evaluation}\label{exp_apndx}

\textbf{Effect of $\mathcal{KP}$ on Efficient KGE Methods Evaluation:}
The research community has recently proposed several KGE methods to improve training efficiency  \cite{wang2021lightweight,wang2022swift,peng2021highly}. Our idea in this experiment is to perceive if efficient KGE methods improve their overall carbon footprint using $\mathcal{KP}$. For the same, we selected state-of-the-art efficient KGE methods: Procrustes \cite{peng2021highly} and HalE \cite{wang2022swift}. Figure \ref{fig:carbon_footprintefficient} illustrates that using $\mathcal{KP}$ for evaluation drastically reduces the carbon footprints of already efficient KGE methods. For instance, the carbon footprint of HalE is reduced from 110g (using hits@10) to 20g of CO2 (using $\mathcal{KP}$). 

\begin{figure} [ht]
    \centering
    \includegraphics[width=0.35\textwidth]{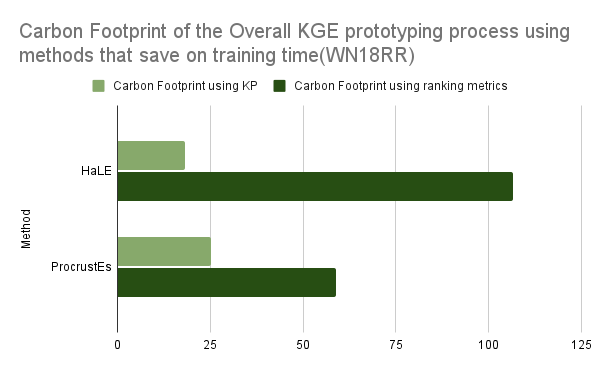}
    \caption{Study on efficient KGE methods and their the carbon footprint on WN18RR when using $\mathcal{KP}$ vs Hits@10. The x-axis shows the the carbon footprint in g eq $CO_2$.}
    \label{fig:carbon_footprintefficient}
\end{figure}

\noindent \textbf{Robustness and Efficiency on large KGs:}
 This ablation study aims to gauge the correlation behavior of $\mathcal{KP}$ and ranking metric on a large-scale KG. For the experiment, we use Yago3-10 dataset. A key reason to select the Yago-based dataset is that besides being large-scale, it has rich semantics. Results in Table \ref{tab:yago_kp_corr} illustrate $\mathcal{KP}$ shows a stable and high correlation with the ranking metric, confirming the robustness of $\mathcal{KP}$. We show carbon footprint results in Figure \ref{fig:carbon_footprint_yago} for the yago dataset. Further we also study the efficiency of $\mathcal{KP}$ on the wikidata dataset \cite{wikidata_wang2021kepler} in Figure \ref{fig:carbon_footprint_wikidata} which reaffirms that $\mathcal{KP}$ maintains its efficiency on large scale datasets.
\\

\noindent \textbf{Efficiency comparison of Sliced Wasserstein vs Wasserstein as distance metric in $\mathcal{KP}$:}
In this study we empirically provide a rationale for using sliced wasserstein as a distance metric over the wasserstein distance in $\mathcal{KP}$. The results are in table \ref{tab:time_analysis_kp_wass_sw}. We see that $\mathcal{KP}$ using sliced wasserstein distance provides a significant computational advantage over wasserstein distance, while having a good performance as seen in the previous experiments. Thus we need an efficient approximation such as the sliced wasserstein distance as the distance metric in place of wasserstein distance in $\mathcal{KP}$.

\begin{table}[ht!]
\begin{adjustbox}{width=0.5\textwidth,center}
  \begin{tabular}{l|c|c|c|c|c}
    \toprule
    \textbf{Metrics} & \textbf{Hits@1}($\uparrow$) & \textbf{Hits@3}($\uparrow$) & \textbf{Hits@10}($\uparrow$) & \textbf{MR}($\downarrow$) & \textbf{MRR}($\uparrow$) \\
    \hline
    r      &0.657& 0.594& 0.414& -0.920& 0.572 \\
    $\rho$ &0.679& 0.679& 0.5& -0.714& 0.643 \\
    $\tau$ &0.524& 0.524& 0.333& -0.524& 0.429 \\
    \bottomrule
  \end{tabular}
  \end{adjustbox}
  \caption{$\mathcal{KP}$ correlations on the YAGO dataset. }
    \label{tab:yago_kp_corr}
\end{table}

\begin{table}[ht!]
\vspace{-0.35cm}
\begin{adjustbox}{width=0.5\textwidth,center}
  \begin{tabular}{|l|c|c|c|c|c|c|}
    \toprule
    \multicolumn{1}{|l|}{\textbf{Metrics}} &
      \multicolumn{2}{c|}{\textbf{$\mathcal{KP}$(W)}} &
      \multicolumn{2}{c|}{\textbf{$\mathcal{KP}$(SW)}} &
      \multicolumn{2}{c|}{\textbf{Speedup $\uparrow$}}\\
      \hline
   \textbf{Dataset} & \textit{FB15K237} & \textit{WN18RR} & \textit{FB15K237} & \textit{WN18RR} & \textit{FB15K237} & \textit{WN18RR} \\
    \hline
    \textbf{split} & \textit{val + test} & \textit{val + test}& \textit{val + test} & \textit{val + test} & \textit{Avg} & \textit{Avg} \\
    \hline
    TransE & 1136.766 & 9.655 & \cellcolor{asparagus!40}0.321 & \cellcolor{asparagus!40}0.114 & x 3540.3 & x 84.6 \\
    TransH & 2943.869 & 7.549 & \cellcolor{asparagus!40}0.317 & \cellcolor{asparagus!40}0.095 & x 9278.5 & x 79.8 \\
    TransR & 1734.576 & 4.423 & \cellcolor{asparagus!40}0.336 & \cellcolor{asparagus!40}0.129 & x 5168.3 & x 34.4 \\
    Complex & 1054.721 & 13.089 & \cellcolor{asparagus!40}0.324 & \cellcolor{asparagus!40}0.135 & x 3255.3 & x 97.0 \\
    RotatE & 865.417 & 12.783 & \cellcolor{asparagus!40}0.342 & \cellcolor{asparagus!40}0.136 & x 2531.1 & x 94.3 \\
    TuckER & 1021.649 & 3.840 & \cellcolor{asparagus!40}0.316 & \cellcolor{asparagus!40}0.098 & x 3230.0 & x 39.1 \\
    ConvKB & 719.310 & 5.154 & \cellcolor{asparagus!40}0.429 & \cellcolor{asparagus!40}0.132 & x 1675.7 & x 39.0 \\
    \bottomrule
  \end{tabular}
  \end{adjustbox}
    \caption{Evaluation Metric Comparison wrt Computing Time (in minutes, for 100 epochs). Column 1 denotes popular KGE methods. Depicted values denote evaluation(validation+test) time for computing a metric and corresponding speedup using $\mathcal{KP} (SW)$. $\mathcal{KP} (SW)$ with sliced wasserstein as the distance metric significantly reduces the evaluation time (green) in comparison with $\mathcal{KP} (W)$ which uses the wasserstein distance.}
    \label{tab:time_analysis_kp_wass_sw}
    \vspace{-0.3cm}
\end{table}

\subsection{Theoretical Proof Sketches}\label{theory_apndx}
We work under the following considerations:
As the KGE method converges the mean statistic($m_{\nu}$) of the scores of the positive triples consistently lies on one side of the half plane formed by the mean statistic($m_{\mu}$) of the negative triples, irrespective of the data distribution. The detail proofs are \href{https://github.com/ansonb/Knowledge_Persistence/blob/main/WWW2023_KP_unlimited_appendix.pdf}{\textbf{here}}.
\begin{lemma}\label{lm_corr_kp_perm_apndx}
    $\mathcal{KP}$ has a monotone increasing correspondence with the Proxy of the Expected Ranking Metrics(PERM) under the above stated considerations as $m_{\nu}$ deviates from $m_{\mu}$
\end{lemma}
\begin{proof}[Proof Sketch]
    Considering the 0-dimensional PD as used by $\kp$ and a normal distribution for the edge weights (can be extended to other distributions using techniques like \cite{sakia1992box}) of the graph(scores of the triples), we have univariate gaussian measures \cite{sharma2018towards} $\mu$ and $\nu$ for the positive and negative distributions respectively. Denote by $m_{\mu}$ and $m_{\nu}$ the means of the distributions $\mu$ and $\nu$ respectively and by $\Sigma_{\mu}$, $\Sigma_{\nu}$ the respective covariance matrices. 
    \begin{equation}
        W_2^2(\mu,\nu) = \norm{\mu-\nu}^2 + B(\Sigma_{\mu},\Sigma_{\nu})^2
    \end{equation}
    where $B(\Sigma_{\mu},\Sigma_{\nu})^2 = tr(\Sigma_{\mu} + \Sigma_{\nu} - 2(\Sigma_{\mu}^{\frac{1}{2}}\Sigma_{\nu}\Sigma_{\mu}^{\frac{1}{2}})^{\frac{1}{2}})$. 

    Next we see how that changing the means of the distribution(and also variance) changes PERM and $\kp$. We can show that,
    \begin{align*}
        P &= \int_{x=-\infty}^{x=\infty} D^{+}(x) \left(\int_{y=x}^{y=\infty} D^{-}(x) dy\right) dx \\
        \frac{\partial P}{\partial m_{\nu}} &= \int_{x=-\infty}^{x=\infty} D^{+}(x) \left(\int_{y=x}^{y=\infty} \frac{\partial D^{-}(x)}{\partial m_{\nu}} dy\right) dx \\
        &\geq 0 \\
        \frac{\partial P}{\partial \Sigma_{\nu}} &= \int_{x=-\infty}^{x=\infty} D^{+}(x) \left(\int_{y=x}^{y=\infty} \frac{\partial D^{-}(y)}{\partial \Sigma_{\nu}} dy \right) dx \\
        &\leq 0\\
        \frac{\partial P}{\partial \Sigma_{\mu}} &= \int_{x=-\infty}^{x=\infty} \frac{\partial D^{+}(x)}{\partial \Sigma_{\nu}} \left(\int_{y=x}^{y=\infty} D^{-}(y) dy\right) dx \\
    \end{align*}
    Since $\kp$ is the (sliced) wasserstein distance between PDs we can show the respective gradients are as below,
     \begin{align*}
        \frac{\partial W_2^{2}(\mu,\nu)}{\partial m_{\nu}} &= 2 |m_{\mu}-m_{\nu}| \\
        &\geq 0 \\
        \frac{\partial W_2^{2}(\mu,\nu)}{\partial \Sigma_{\nu}} &= I - \Sigma_{\mu}^{\frac{1}{2}} (\Sigma_{\mu}^{\frac{1}{2}} \Sigma_{\nu} \Sigma_{\mu}^{\frac{1}{2}})^{\frac{-1}{2}} \Sigma_{\mu}^{\frac{1}{2}} \\
    \end{align*}
    As the generating process of the scores changes the gradient of PERM along the direction $(dm_{\nu}, d\sigma_{\mu}, d\sigma_{\nu})$ can be shown to be the following
    \begin{align*}
        &\left< \left(dm_{\nu}, d\sigma, d\sigma\right), \left(\frac{\partial PERM}{\partial m_{\nu}}, \frac{\partial PERM}{\partial \Sigma_{\mu}}, \frac{\partial PERM}{\partial \Sigma_{\nu}}\right) \right> \\
        &\geq 0
    \end{align*}

 Similarly the gradient of $\kp$ along the direction $(dm_{\mu}, d\sigma_{\mu}, d\sigma_{\nu})$ is
    \begin{align*}
        &\left< (dm_{\nu}, d\sigma, d\sigma), (\frac{\partial W_2^{2}(\mu,\nu)}{\partial m_{\nu}}, \frac{\partial W_2^{2}(\mu,\nu)}{\partial \Sigma_{\mu}}, \frac{\partial W_2^{2}(\mu,\nu)}{\partial \Sigma_{\nu}}) \right> \\
        &\geq 0 \\
    \end{align*}

    Since both PERM and and $\kp$ vary in the same manner as the distribution changes, the two have a one-one correspondence \cite{spearman_10.2307/1412408}. 
\end{proof}

The above lemma shows that there is a one-one correspondence between $\kp$ and PERM and by definition PERM has a one-one correspondence with the ranking metrics. Therefore, the next theorem follows as a natural consequence
\begin{theorem}\label{thm_corr_kp_ranking_metrics_apndx}
    $\mathcal{KP}$ has a one-one correspondence with the Ranking Metrics under the above stated considerations
\end{theorem}
    

\begin{theorem}\label{thm_stability_kp_ranking_apndx}
    Under the considerations of theorem \ref{thm_corr_kp_ranking_metrics_apndx}, the relative change in $\mathcal{KP}$ on addition of random noise to the scores is bounded by a function of the original and noise-induced covariance matrix as
    $\frac{\Delta \mathcal{KP}}{\mathcal{KP}} \leq max( (1 - |\Sigma_{\mu_1}^{+1} \Sigma_{\mu_2}^{-1}|^{\frac{3}{2}}), (1 - |\Sigma_{\nu_1}^{+1} \Sigma_{\nu_2}^{-1}|^{\frac{3}{2}}) )$,
    where $\Sigma_{\mu_1}$ and $\Sigma_{\nu_1}$ are the covariance matrices of the positive and negative triples' scores respectively and $\Sigma_{\mu_2}$ and $\Sigma_{\nu_2}$ are that of the corrupted scores.
\end{theorem}
\begin{proof}[Proof Sketch]
    Consider a zero mean random noise to simulate the process of varying the distribution of the scores of the KGE method. Let $m_{\mu_1}$ and $m_{\nu_1}$ be the means of the positive and negative triples' scores of the original method and $\Sigma_{\mu_1}$, $\Sigma_{\nu_1}$ be the respective covariance matrices. Let  $m_{\mu_2}$ and $m_{\nu_2}$ be the means of the positive and negative triples' scores of the corrupted method and $\Sigma_{\mu_2}$, $\Sigma_{\nu_2}$ be the respective covariance matrices. 
    Considering the kantorovich duality \cite{villani2009optimal} and taking the difference between the two measures we have
    \begin{align*}
        &\kp_1 - \kp_2 \\
        &= \underset{\gamma_1 \in \Pi(x,y)}{inf} \int_{\gamma_1} Distance(x,y) d\gamma_1(x,y) \\
        &- \underset{\gamma_2 \in \Pi(x,y)}{inf} \int_{\gamma_2} Distance(x,y) d\gamma_2(x,y)\\ 
        &\leq \underset{\Phi,\Psi}{sup} \int_{x} \Phi(x) d\mu_1(x) + \int_{y} \Psi(y) d\nu_1(y) \\
        &- \int_{x} \Phi(x) d\mu_2(x) - \int_{y} \Psi(y) d\nu_2(y) \\
        &\leq \underset{\Phi,\Psi}{sup} \int_{x} \Phi(x) (d\mu_1(x) - d\mu_2(x)) + \int_{y} \Psi(y) (d\nu_1(y) - d\nu_2(y)) \\
    \end{align*}
    Now by definition of the measure $\mu_1$ we have
    \begin{align*}
        \frac{\partial \mu_1}{\partial x} &= -\mu_1 \Sigma_{\mu_1}^{-1} (x-m_{\mu_1}) \\
        d\mu_1(x_i) &= -(\mu_1 \Sigma_{\mu_1}^{-1} (x-m_{\mu_1}))[i] dx_i \\
        \therefore d\mu_1(x) &= det( diag(-\mu_1 \Sigma_{\mu_1}^{-1} (x-m_{\mu_1}))) dx \\
    \end{align*}
    From the above results we can show the following
    \begin{align*}
        &\kp_1 - \kp_2 \\
        &\leq max((1 - det(\Sigma_{\mu_1} \Sigma_{\mu_2}^{-1})^{\frac{n}{2}+1} ),(1 - det(\Sigma_{\nu_1} \Sigma_{\nu_2}^{-1})^{\frac{n}{2}+1} )) \kp_1 \\
        &\therefore \frac{\Delta \kp}{\kp} \leq  max\left(\left(1 - det(\Sigma_{\mu_1} \Sigma_{\mu_2}^{-1})^{\frac{n}{2}+1} \right),\left(1 - det(\Sigma_{\nu_1} \Sigma_{\nu_2}^{-1})^{\frac{n}{2}+1} \right)\right)
    \end{align*}
    In our case as we work in the univariate setting $n=1$ and thus we have $ \frac{\Delta \kp}{\kp} \leq  max\left(\left(1 - det(\Sigma_{\mu_1} \Sigma_{\mu_2}^{-1})^{\frac{3}{2}} \right),\left(1 - det(\Sigma_{\nu_1} \Sigma_{\nu_2}^{-1})^{\frac{3}{2}} \right)\right)$, as required.

\end{proof}
Theorem \ref{thm_stability_kp_ranking_apndx} shows that as noise is induced gradually, the $\mathcal{KP}$ value changes in a bounded manner as desired.

\end{document}